\documentclass{article}

\usepackage{microtype}
\usepackage{subfigure}
\usepackage{graphicx}
\usepackage{enumitem}
\usepackage{booktabs} %
\usepackage{wrapfig}
\usepackage{bm}
\usepackage{listings}
\usepackage{xspace}

\newcommand{\p}{\pi_{\text{ref}}}
\newcommand{\EOS}{\textit{\footnotesize{EOS}}\xspace}
\newcommand{\by}{\mathbf{y}}
\newcommand{\bz}{\mathbf{z}}
\newcommand{\bx}{\mathbf{x}}
\newcommand{\bw}{{\bm{\theta}}}
\newcommand{\KL}{\textit{KL\hspace{0.01in}}}
\newcommand{\mc}{\mathcal}

\usepackage{hyperref}

\usepackage[accepted]{icml2024}

\usepackage{amsmath}
\usepackage{amssymb}
\usepackage{mathtools}
\usepackage{amsthm}

\usepackage[capitalize,noabbrev]{cleveref}

\theoremstyle{plain}
\newtheorem{theorem}{Theorem}[section]

\newtheorem{lemma}[theorem]{Lemma}

\theoremstyle{definition}

\theoremstyle{remark}

\usepackage[textsize=tiny]{todonotes}

\icmltitlerunning{Controlled Decoding from Language Models}

\begin{document}

\twocolumn[
\icmltitle{Controlled Decoding from Language Models}

\icmlsetsymbol{equal}{*}

\begin{icmlauthorlist}
\icmlauthor{Sidharth Mudgal}{equal,gdm}
\icmlauthor{Jong Lee}{equal,gdm}
\icmlauthor{Harish Ganapathy}{gdm}
\icmlauthor{YaGuang Li}{gdm}
\icmlauthor{Tao Wang}{openai}
\icmlauthor{Yanping Huang}{gdm}
\icmlauthor{Zhifeng Chen}{gdm}
\icmlauthor{Heng-Tze Cheng}{gdm}
\icmlauthor{Michael Collins}{gdm}
\icmlauthor{Trevor Strohman}{gdm}
\icmlauthor{Jilin Chen}{gdm}
\icmlauthor{Alex Beutel}{openai}
\icmlauthor{Ahmad Beirami}{gdm}
\end{icmlauthorlist}

\icmlaffiliation{gdm}{Google DeepMind}
\icmlaffiliation{openai}{OpenAI (work done at Google)}

\icmlcorrespondingauthor{\\Sidharth Mudgal}{sidharthms@google.com}
\icmlcorrespondingauthor{Jong Lee}{leejong@google.com}
\icmlcorrespondingauthor{Ahmad Beirami}{beirami@google.com}

\icmlkeywords{Machine Learning, ICML}

\vskip 0.3in
]

\printAffiliationsAndNotice{\icmlEqualContribution} %

\begin{abstract}
KL-regularized reinforcement learning (RL) is a popular alignment framework to control the language model responses towards high reward outcomes. We pose a tokenwise RL objective and propose a modular solver for it, called {\em controlled decoding (CD)}. CD exerts control through a separate {\em prefix scorer} module, which is trained to learn a value function for the reward. The prefix scorer is used at inference time to control the generation from a frozen base model, provably sampling from a solution to the RL objective. We empirically demonstrate that CD is effective as a control mechanism on popular benchmarks. We also show that prefix scorers for multiple rewards may be combined at inference time, effectively solving a multi-objective RL problem with no additional training. We show that the benefits of applying CD transfer to an unseen base model with no further tuning as well. Finally, we show that CD can be applied in a blockwise decoding fashion at inference-time, essentially bridging the gap between the popular best-of-$K$ strategy and tokenwise control through reinforcement learning. This makes CD a promising approach for alignment of language models.
\end{abstract}

\vspace{-.2in}
\section{Introduction}
\vspace{-.06in}
Generative language models have reached a level where they can effectively  solve a variety of open-domain tasks with little task specific supervision. Hence, it is crucial to ask: {\em how can we align machine generated content to rewards when we have no control over the pre-trained representations in a generative language model?}

Controlling language model responses towards high reward outcomes is an area of active research in the literature. We divide the existing alignment methods into two  categories that differ significantly in real-world deployment: {\em generator improvement} and {\em inference-time add-on} solutions. 

Generator improvement solutions, such as KL-regularized PPO~\citep{christiano2017deep, ouyang2022training}, direct preference optimization (DPO)~\citep{rafailov2023direct}, sequence likelihood calibration (SliC)~\citep{zhao2022calibrating}, and identity preference optimization (IPO)~\cite{azar2023general} update the weights of the language model to align it with a reward model. They are efficient for inference but offer little configurability on the reward.

A simple and effective inference-time add-on solution is best-of-$K$~\citep{nakano2021webgpt,stiennon2020learning,touvron2023llama}, where $K$ i.i.d. samples are drawn from a base model,  ranked based on a reward, and the highest ranking one is selected. 
Other methods, such as  
FUDGE~\citep{yang-klein-2021-fudge} or COLD~\citep{qin2022cold}, offer a prefix scorer that is used at inference-time to control a frozen base model response towards high-reward outcomes. 
Due to their modularity of design which leaves the base model frozen, these methods offer inference-time configurability. Our goal is to propose a learning framework for such methods. 

Our contributions are summarized below. \vspace{-.1in}
\begin{itemize}[leftmargin=*, itemsep=-0.02in]
    \item We formalize a modular alignment method, {\em controlled decoding (CD)}, to solve a KL-regularized RL objective. CD learns a prefix  scorer for the reward that is used to steer the generation from a partially decoded path. 
    
    \item We show that two variants of CD, namely CD-FUDGE~\citep{yang-klein-2021-fudge} and CD-Q (ours), provably lead to sampling from a solution to the RL objecive.
    
      \item We propose {\em blockwise CD} where the prefix scorer is used to select the best-of-$K$ paths for a decoded block of $M$ tokens. This bridges the gap between the sequence-level best-of-$K$ and tokenwise RL methods.
    
    \item We empirically show that CD offers significant improvement over existing controlled generation/decoding solutions on popular benchmarks.
    
    \item We show that CD prefix scorer transfers to an unseen base model with no further training.
    
    \item We demonstrate the modularity of CD at inference-time to integrate multiple rewards into a single prefix scoring rule, and applying it to an unseen base model.
\end{itemize}

\vspace{-.2in}
\section{KL-Regularized Reinforcement Learning}
\vspace{-.05in}
Let $\bx$ be a prompt (consisting of several tokens) and let $\by = y^T := [y_1, \ldots, y_T]$ represent a response that is a concatenation of $T$ tokens. Here each token $y_t \in \mathcal{Y}$, where $\mathcal{Y}$ represents the alphabet (vocabulary).  Let $\p$ denote a pre-trained language model (LM) that is used to draw samples in an autoregressive manner. In particular, we use $\p(\cdot|[\bx, y^t])$ to denote the distribution that the LM induces on the next token on alphabet $\mathcal{Y}$ given the input that is the concatenation of the prompt $\bx$ and a partially decoded response $y^t$ of $t$ tokens.
Let $r([\bx, \by])$ be a scalar valued reward function bounded from above, e.g., the log-likelihood of a scoring function for the event that the response $\by$ in context $\bx$ is deemed safe. We define the following tokenwise reward:\vspace{-.06in}
\begin{equation*}
    R([\bx, y^{t}]):= \left\{\begin{array}{ll}
       0  & \quad \quad y_t \neq \EOS\\
  r([\bx, y^t]) & \quad \quad y_t = \EOS
        \end{array}\right., 
        \vspace{-.06in}
\end{equation*}
where $\EOS$ represents the end of sequence. Here, we only give a reward once decoding has completed and otherwise no reward is assigned to a decoding path.
We then define the {\em value function} associated with the reward as:\vspace{-.1in}
\begin{equation}
    V^\star([\bx, y^t]) := E_{z_1, z_2, \ldots \sim \p} \left\{ \sum_{\tau \geq 0}  R([\bx, y^t, z^\tau])\right\}.
\label{eq:value-definition}\vspace{-.06in}
\end{equation}
The value function captures the expected cumulative reward of a fully decoded response when decoding continues from a partially decoded sequence $y^t,$ using the base language model $\p.$

For a given $[\bx, y^t]$ such that $y_t \neq \EOS,$ we define the advantage function of a decoding policy $\pi$ as:
\begin{align*}
    A([\bx, y^{t}]; \pi) \hspace{-.02in}& \hspace{-.03in}:=\hspace{-.01in}
       E_{z \sim \pi} \left\{ V^\star([\bx, y^t, z]) - V^\star([\bx, y^{t}])\right\} \nonumber\\
       & \hspace{-.05in}= \hspace{-.03in}\sum_{z \in \mathcal{Y}} \pi(z|[\bx, y^t]) V^\star([\bx, y^t, z]) - V^\star([\bx, y^t]).
\end{align*}
Note that the advantage of the base policy is given by $A([\bx, y^{t}]; \p) = 0$ (law of total probability), and hence our goal is to choose $\pi$ to deviate from $\p$ to achieve a positive advantage over the base policy. 

Let $D([\bx, y^t];\pi)$ be the tokenwise KL divergence between a decoding policy $\pi$ and a frozen base language model $\p$ for decoding the next token after $[\bx, y^t]$ for $y_t \neq \EOS$:
\begin{align*}
   D([\bx, y^t]; \pi) &:= \KL(\pi(\cdot|[\bx,y^t] ) \| \p(\cdot|[\bx, y^t])) \nonumber\\&=  
    \sum_{z \in \mathcal{Y}} \pi(z|[\bx, y^t]) \log \left( \frac{\pi(z|[\bx, y^{t}])}{\p(z |[\bx, y^{t}])}\right),
\end{align*}
where $\KL(\cdot\|\cdot)$ denotes the KL divergence (also known as relative entropy). Recall that our goal is not to deviate too much from the base policy (measured in KL divergence) because that is expected to lead to the degeneration of the language model in other top-line performance metrics. 

To satisfy these conflicting goals, we use the KL-regularized RL objective which is defined as: %
\begin{align}
    J_\lambda([\bx, y^t]; \pi) &:= \lambda  A ([\bx, y^t]; \pi) - D([\bx, y^t]; \pi) ,
\label{eq:RL-KL}
\end{align}
where $\lambda \in \mathbb{R}^{\geq 0}$ trades off reward for drift from the base language model. Note that $J_\lambda([\bx, y^t]; \pi)$ is concave in $\pi.$ This is because $A ([\bx, y^t]; \pi)$ is linear in $\pi$ and $D([\bx, y^t]; \pi)$ is convex in $\pi.$
The first term denotes the advantage term for the reward that will be eventually obtained once the response is fully decoded. The second term is a language model (LM) negative reward signal penalizing the policy $\pi$ for drifting too far from the initial policy $\p$. 

We let $\pi_\lambda^\star(z|[\bx, y^{t}])$ denote the decoding policy function that maximizes \eqref{eq:RL-KL}.
Note that at the extreme of $\lambda=0,$ we have
$
    \pi_0^\star(z|[\bx, y^{t}]) = \p(z|[\bx, y^t])
$
which achieves $D([\bx, y^t]; \p) = 0$ and $A([\bx, y^t]; \p) = 0$. We are interested in characterizing the tradeoff curves between $A$ and $D$ achieved by $\lambda \in \mathbb{R}^{\geq 0}$ to increase $A([\bx, y^t]; \pi)$ at the cost of an increased KL penalty, $D([\bx, y^t]; \pi)$.
Our main result in this section is the following characterization of $\pi_\lambda^\star.$ 

\begin{theorem}
The optimal policy for the RL objective is unique and is given by
\begin{equation}
    \pi_\lambda^\star(z|[\bx, y^{t}]) \propto  p(z|[\bx , y^{t}]) e^{\lambda V^\star([\bx, y^t, z])}.
\end{equation}
\label{thm:RL-solution}
\vspace{-.2in}
\end{theorem}
This result resembles that of~\cite{korbak2022rl}, with the main difference being the controller is tokenwise here.
Recall that our goal is to develop an inference-time alignment solution that keeps the language model frozen. Theorem~\ref{thm:RL-solution} gives us a way to do that by combining logits from a frozen LM and those of a value function.

{\bf Remark.} The tokenwise RL formulation here is more restrictive than the sequence-level RL, used to design RLHF and DPO.  However, we will compare with them on sequence-level {\em expected reward} vs {\em KL} tradeoffs. 

\vspace{-.07in}
\section{Controlled Decoding}
Our goal is to learn $V_\bw([\bx, y^t])$ parameterized by $\bw$ to match $V^\star([\bx, y^t])$ through the following $L_2$ objective function:\footnote{It may be possible to devise a more effective distillation objective through Fisher information shaping or other divergences.}
\begin{align*}
 &\mc{L}^\star(\bw) = E_{\bx \sim \mu} E_{\by \sim \p(\cdot | \bx)} \ell^\star(\bx, \by; \bw),\\
 \text{where}~~   &\ell^\star(\bx, \by; \bw) = \frac{1}{2} \sum_{t \in [|\by|]}(V_\bw([\bx, y^t]) - V^\star([\bx, y^t]))^2,
\end{align*}
where $\mu$ is a distribution over training prompts. Next, we present two methods to learn the prefix scorer, and two ways to use it at inference time for control.

\subsection{Training the prefix scorer}

{\bf CD-FUDGE~\cite{yang-klein-2021-fudge}.}
Given $\bx \sim \mu,$ let $\by = ([y_1, \ldots, y_T])$ be a stochastic draw from the base model $\p$.  Consider $r([\bx, \by])$  to be the stochastic reward of the fully decoded completion, $\by$. Let
\begin{align}
&\mc{L}_{F}(\bw) = E_{\bx \sim \mu} \ell_{F}(\bx, \by; \bw),\quad \text{s.t.}~~ \by \sim \p,\label{eq:fudge}\\
  \text{where}~~ &\ell_{F}(\bx, \by; \bw) = \frac{1}{2}\sum_{t \in [|\by|]} \left(V_\bw([\mathbf{x}, y^t]) - r([\bx, \by]) \right)^2.\nonumber
\end{align}

Now we state our main result on CD-FUDGE, which is formally stated and proved in Appendix~\ref{app:proof}, Theorem~\ref{thm:formal-fudge-convergence}.
\begin{theorem}[informal]
Under regularity assumptions, SGD on $\mc{L}_F$ converges to a stationary point of $\mc{L}^\star(\bw).$
\end{theorem}
This is a remarkable result. It states that if the dataset used for training the prefix scorer in FUDGE~\citep{yang-klein-2021-fudge} is obtained by rolling out the base model, then FUDGE prefix scorer may be used to solve the RL problem in Eq.~\eqref{eq:RL-KL}. Next, we state our proposal which is an off-policy solver without the need for rolling out the base model.

{\bf CD-Q.}
Notice the following Bellman identity~\citep{sutton2018reinforcement}:
\begin{align*}
V^\star([\mathbf{x}, y^{t}]) = 
    & \left\{\begin{array}{ll} \hspace{-.05in}
 E_{z \sim \p (\cdot|[x,y^{t}])}  V^\star([\mathbf{x},y^{t}, z]), & y_t \hspace{-.01in} \neq \hspace{-.01in} \EOS\\
     r([\mathbf{x}, y^t]), & y_t \hspace{-.01in}= \hspace{-.01in}\EOS
    \end{array}\right. \hspace{-.05in}.
\end{align*}
We present a simple solution to train a prefix scorer.
Inspired by the policy evaluation updates in DQN~\citep{mnih2013playing}, we optimize the following loss function:
\begin{align}
  &\mc{L}_{Q}(\bw) = E_{\bx \sim \mu} \ell_{Q}(\bx, \by; \bw),\label{eq:Q-learning}\\
  \text{where}~ & \ell_Q(\mathbf{x}, y^t; \bw) \hspace{-.02in}=\hspace{-.02in} \frac{1}{2} \sum_{t \in [|\by|]}\left(V_\bw([\mathbf{x}, y^t]) - \dot{v}_t \right)^2\hspace{-.06in}, \nonumber\\
      & \hspace{-.25in} v_t \hspace{-.01in}=\hspace{-.01in} \left\{\hspace{-.05in}\begin{array}{ll}
 \sum_{z \in \mathcal{Y}} \p (z|[x,y^{t}]) V_\bw([\mathbf{x},y^{t}, z])  &  ~y_t \neq  \EOS\\
     r([\mathbf{x}, y^t]) &  ~ y_t = \EOS
    \end{array}\right. \hspace{-.07in}, \nonumber
\end{align}
and where $\dot{v}$ implies a stop gradient over $v$ (even though it inherently depends on $\bw$).

The abovementioned learning procedure for the prefix scorer may be performed over an {\em off-policy} dataset, scored offline using the reward for all $[\bx, \by]$~\citep{sutton2018reinforcement}. 
On the other hand, training the prefix scorer requires  (on-demand) access to the base language model $\p$ to compute the target $v_t$ in~\eqref{eq:Q-learning}. A simple modification of this procedure can be shown to be provably convergent~\citep{wang2021convergent}.\footnote{Note that one may improve on the proposed solver (cf.~\citep{hessel2018rainbow}), but we present the simplest form for the sake of clarity, which already gives good empirical performance. } We also remark that many other improvements over DQN have been proposed over the years, many of which amount to Rainbow~\citep{hessel2018rainbow}. Exploring how to improve CD-Q using these techniques is an interesting are for future work. 

\subsection{Inference-time sampling strategies}
\label{sec:inference-time-sampling}
Equipped with the prefix scorer, we use it in two different ways at inference time to align the base model.

\vspace{-.05in}
\paragraph{Tokenwise sampling.}
We use the prefix scorer for tokenwise sampling per Theorem~\ref{thm:RL-solution}. In this case, given context $\bx$ and a partially decoded sequence $y^t,$ we obtain the logits of $\p([\bx, y^t,z])$ and $V_\bw([\bx, y^t, z])$ for all $z$ from the base policy and the prefix scorer. Then, we linearly combine the logits to sample from the following distribution:
\begin{align}
    & z \sim \pi_\bw(\cdot|[\bx, y^t]) \\
    \text{where}~~~ &\pi_\bw(z|[\bx, y^t])  \propto \p(z|[\bx, y^t]) e^{\lambda V_\bw([\bx, y^t, z])}. \nonumber
\end{align}
An illustration of tokenwise sampling using CD prefix scorer is presented in Figure~\ref{fig:tokenwise-CD}, where the prefix scorer is used to downweight decoding of tokens that may lead to undesirable outcomes. Note that tokenwise sampling is the most straight-forward way to use the prefix scorer, which requires one call to the prefix scorer per decoding of each token, and was also used by~\citet{yang-klein-2021-fudge}.%

\begin{figure}[h]
    \centering
    \includegraphics[width=\linewidth]{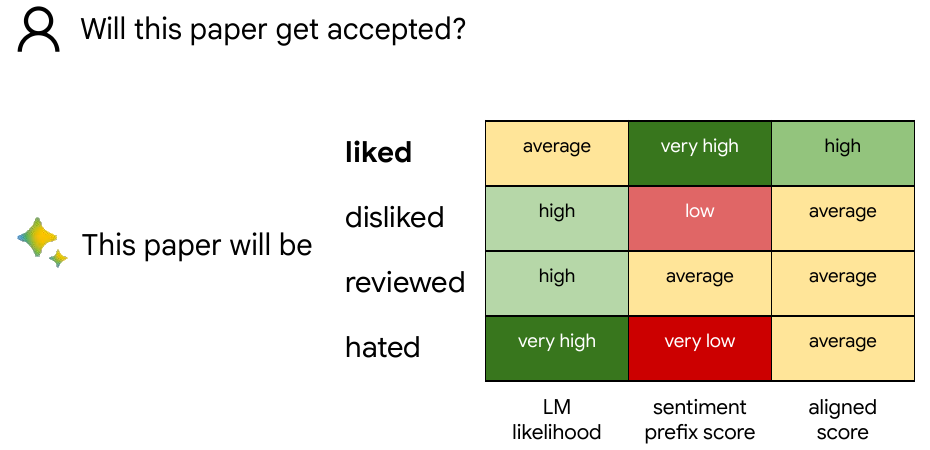}
    \vspace{-.2in}
    \caption{An illustration of {\bf tokenwise sampling} using CD prefix scorer where the alignment goal is to decode sequences with positive sentiment. The sentiment score is used to shape the overall {\em aligned score} for sampling, which results in downweighting of the high likelihood tokens that might result in negative sentiment and upweighting of tokens that lead to positive sentiment. }
    \label{fig:tokenwise-CD}
    \vspace{-.1in}
\end{figure}

\vspace{-.05in}
\paragraph{Blockwise best-of-$K$.}
Next, we present a sampling strategy that combines RL with best-of-$K$.
We sample $K$ i.i.d. continuation blocks of length $M$ from the base policy, and accept the continuation with the highest prefix score and reject the rest:
\begin{align}
\label{eq:blockwise-bok}
    &z^M := \arg \max_{ \left\{z_{(k)}^M\right\}_{k \in [K]}} V_\bw([\bx, y^t, z^M_{(k)}])\\
    \text{where}~~~ &\left\{z_{(k)}^M\right\}_{k \in [K]} \overset{\text{i.i.d.}}{\sim} \p(z^M|[\bx, y^t]), \nonumber
\end{align}
and continue until a candidate with \EOS has been accepted. 

An illustration of the blockwise sample and rerank is presented in Figure~\ref{fig:blockwise-CD}, where the prefix scorer is used to rerank $M$(=4) decoding paths and choose the candidate with the most positive sentiment.

\begin{figure}[t]
    \centering
    \includegraphics[width=\linewidth]{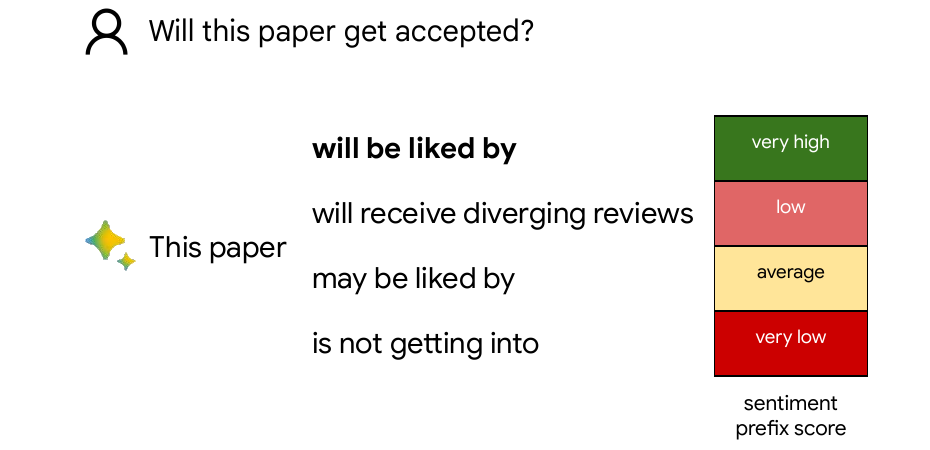}
    \vspace{-.2in}
    \caption{An illustration of {\bf blockwise best-of-$K$} using CD prefix scorer where the alignment goal is to decode sequences with positive sentiment. First, $K$(=4) continuations of length $M$(=4) tokens are sampled from the base LM, and scored using the prefix scorer. The block of tokens with the highest prefix score is selected as the continuation, and the process is continued. }
    \label{fig:blockwise-CD}
    \vspace{-.15in}
\end{figure}

{\bf Blockwise vs tokenwise control.} Note that similar to best-of-$K,$ blockwise CD is not designed to optimally solve the sequence level KL-regularized objective that is the objective of RLHF methods, such as PPO and DPO. However, empirically we observe that best-of-$K$ often results in better reward-KL tradeoffs, e.g.,~\citep[Figure 1]{gao2023scaling} and~\citep[Figure 3]{rafailov2023direct}. In fact, best-of-$K$ is shown to be almost sampling from the optimally aligned distribution through KL-regularized RL~\citep{yang2024asymptotics}. This motivates the exploration of blockwise control techniques that rely on the strength of best-of-$K$.

{\bf Blockwise control vs Best-of-$K$.}
In terms of inference throughput, blockwise CD is similar to the best-of-$K$ for the same value of $K$. However, it offers two major advantages: \vspace{-.05in}
\begin{enumerate}[leftmargin=*, itemsep=0in]
    \item The decoding latency here is only $M$ tokens, whereas the best-of-$K$ method needs to fully decoded all $K$ sequences before it can select one to be served. If the sequence length is large, e.g., when the prompt is to {\em write an essay}, this would not be tolerated. This can open up new applications such as streaming.
    \item To achieve high rewards, best-of-$K$ might require unreasonably high values of $K$. Blockwise CD enables similar reward values with significantly smaller $K.$ We experimentally show the same reward level as best-of-$K$ with up to 10x smaller $K.$
    \vspace{-.05in}
\end{enumerate}

\vspace{-.1in}
\section{Experimental Setup}
\vspace{-.05in}

We examine performance of the controlled decoding models with our proposed inference-time sampling strategies across two tasks. For all experiments, unless otherwise specified the base generative model we use is PaLM 2-XXS (Gecko), and the prefix scorer is also finetuned from PaLM 2-XXS. %

\vspace{-.05in}
\subsection{Datasets}

{\bf DSTC8 Reddit conversations corpus~\citep{Reddit}} is a dataset containing millions of multi-turn conversations from Reddit threads. We use this dataset to optimize response length.

{\bf Anthropic HH~\citep{bai2022training}} is a helpfulness and harmlessness benchmark where the assistant tries to complete next turn in a conversation with a human. We use this to train a reward model that learns human preferences on the helpfulness and harmlessness of the generation.

{\bf TL;DR~\citep{stiennon2020learning}} is a dataset of Reddit posts where each example has information about the post, two summarization candidates, and a preference from a human annotator. We use this to train a reward model that learns summarization preference.

\vspace{-.05in}
\subsection{Reward Models}
\vspace{-.05in}
{\bf Response length.} We used the length of the response as a reward. In this case, we used
 $r_{\text{length}}([\bx, y^T]) = \log(T/ T_{\max})$, where $T_{\max}=1024$. 
 
{\bf Helpfulness and harmlessness.}  
We trained a reward model (Reward-XXS) by finetuning PaLM 2-XXS using pairwise preference data of Anthropic HH~\citep{bai2022training} via the Bradley-Terry (BT) model and selected the checkpoint with the highest eval accuracy. 
Here, $r_{\text{HH}}([\bx, y^T])$ is the log-probability of the resulting pointwise HH classifier.

{\bf Summary quality.} 
Similarly, we trained a PaLM 2-XXS reward model using the pairwise preferences on summary quality~\citep{stiennon2020learning} using the BT model, and picked the checkpoint with the highest eval accuracy. %

\vspace{-.05in}
\subsection{Baselines}
\vspace{-.05in}

In addition to CD-Q and blockwise CD-Q, we consider the following baselines.

{\bf CD-FUDGE~\citep{yang-klein-2021-fudge}}  is trained in the same way as CD-Q with the difference being the target in~\eqref{eq:Q-learning} replaced by the explicit reward received in a given decoding path from the dataset. For best performance, CD-FUDGE is trained on a dataset where the responses are obtained by rolling out the base model. Additionally, we also consider the blockwise best-of-$K$ variant of FUDGE~\cite{yang-klein-2021-fudge}, named {\em blockwise CD-FUDGE}, which is inspired by the proposed blockwise CD-Q method in this paper.

{\bf KL-regularized PPO~\citep{ouyang2022training}} solves a KL-regularized RL problem using PPO~\citep{schulman2017proximal}.

{\bf DPO~\citep{rafailov2023direct}} is trained on a pairwise preference dataset. For a more fair comparison, we used {\em online DPO} by rolling out the policy and sampling two generations and optimizing the DPO objective on their explicit rewards.

{\bf IPO~\citep{azar2023general}} is trained in a similar way to DPO except that the objective bakes in new regularization to avoid some of the degeneration issues of DPO. Similarly to DPO, we use {\em online IPO} in this paper.

{\bf Best-of-$K$} is an inference-time alignment solution where $K$ responses are drawn from the base model, ranked using the reward, and the best one is selected.

\vspace{-.05in}
\subsection{Evaluation Metrics}
{\bf KL divergence.} We measure the KL divergence between the aligned policy and the base policy, $E_{\bx \sim \mu} E_{\by \sim  \pi(\cdot | x)} \{ \log \pi(\by|\bx) - \log \p(\by|\bx) \}$, as a proxy for deterioration of model capabilities and reward overoptimization.
For CD-Q and CD-FUDGE, we sweep the strength of the prefix scorer to control $\KL(\pi \| \p).$ For PPO, DPO and IPO, we sweep the strength of the (implicit) KL-regularizer to achieve the same goal. Finally, for best-of-$K,$ blockwise CD-Q, and blockwise CD-FUDGE, we do this by sweeping $K.$ For best-of-$K$, we use the upper bound formula on KL divergence $\KL(\pi \| \p) 
\leq \log(K) - (K-1)/K$~\cite{stiennon2020learning, beirami2024theoretical}. For blockwise sampling strategies, we use an upper bound on the KL divergence given by 
$
    \KL(\pi \| \p) \leq E_{\bx \sim \mu }\left(\log(K) - (K-1)/K \right) \left\lceil \frac{L_{\bx}}{M}\right\rceil,
$
 where $L_\bx$ is the number of decoded tokens in the full response given prompt $\bx$, which is an extension of~\citep[Theorem 1]{beirami2024theoretical}. To this end, we focus on KL values smaller than $10$, beyond which the policy shows significant signs of overfitting~\citep{eisenstein2023helping}. We also remark that the sequence-level KL divergence used here for evaluation is different from our token-level design, which makes the evaluation more favorable to PPO, DPO, and IPO that directly optimize the tradeoff between expected reward and sequence-level KL divergence.

{\bf Normalized expected reward.} We report the expected reward of the aligned policy, $E_{\bx \sim \mu}E_{\by \sim \pi_\bw(\cdot | x)} r(\bx,\by)$, normalized to that of the reference policy.

{\bf Win-rate against base policy.} We report the win-rate of the aligned policy against the base policy, $E_{\bx \sim \mu}E_{\by \sim  \pi_\bw(\cdot | x)} E_{\bz \sim \p(\cdot | x)} \mathbf{1}[r(\bx,\by)> r(\bx,\bz)]$.

{\bf Reward vs KL tradeoffs.} Following~\cite{gao2023scaling}, we report tradeoff curves for {\em reward} vs. {\em KL divergence} between the aligned policy and the base, $\KL(\pi \| \p).$ A method that dominates (i.e., increases the reward with smallest KL budget) is more desirable.

\vspace{-.05in}
\subsection{Training Details}

    {\bf Response length experiments.} Using the Reddit conversations corpus, we used PaLM 2-XXS~\citep{google2023palm} to train prefix scorers and also as the base model for DPO, IPO, and PPO. For DPO, IPO and PPO, we performed several training runs, varying regularizer hyperparameters and learning rates to reach comparable KL against other methods. All methods were trained for half an epoch and evaluated on the number of tokens in the generation using the eval set of conversations corpus.%
    
    {\bf Helpfulness and harmlessness (HH) experiments.} 
    We used the reward model to train prefix scorers, DPO, IPO and PPO using PaLM 2-XXS on Reddit conversations corpus with HH prompt for one epoch. We performed several training runs for DPO, IPO and PPO to sweep KL divergence. Finally, we used PaLM 2-L (Unicorn)~\citep{google2023palm} on the eval set of the conversations corpusto evaluate the helpfulness and harmlessness of the generation. The prompt can be found in Appendix~\ref{app:additional-exp}.

{\bf Summarization experiments.}  
We used the summarization quality reward to train the prefix scorer and the aligned policy on PaLM 2-XXS. For evaluation, we prompted PaLM 2-L (Unicorn)~\citep{google2023palm} on the test set of the TL;DR corpus with to evaluate the summarization quality of the generations compared to vanilla PaLM 2-XXS, and reported the preference win rate. The zeroshot prompt we used to evaluate can be found in Appendix~\ref{app:additional-exp}. %

 \begin{figure}[t]
\centering
\includegraphics[width=0.9215\linewidth]{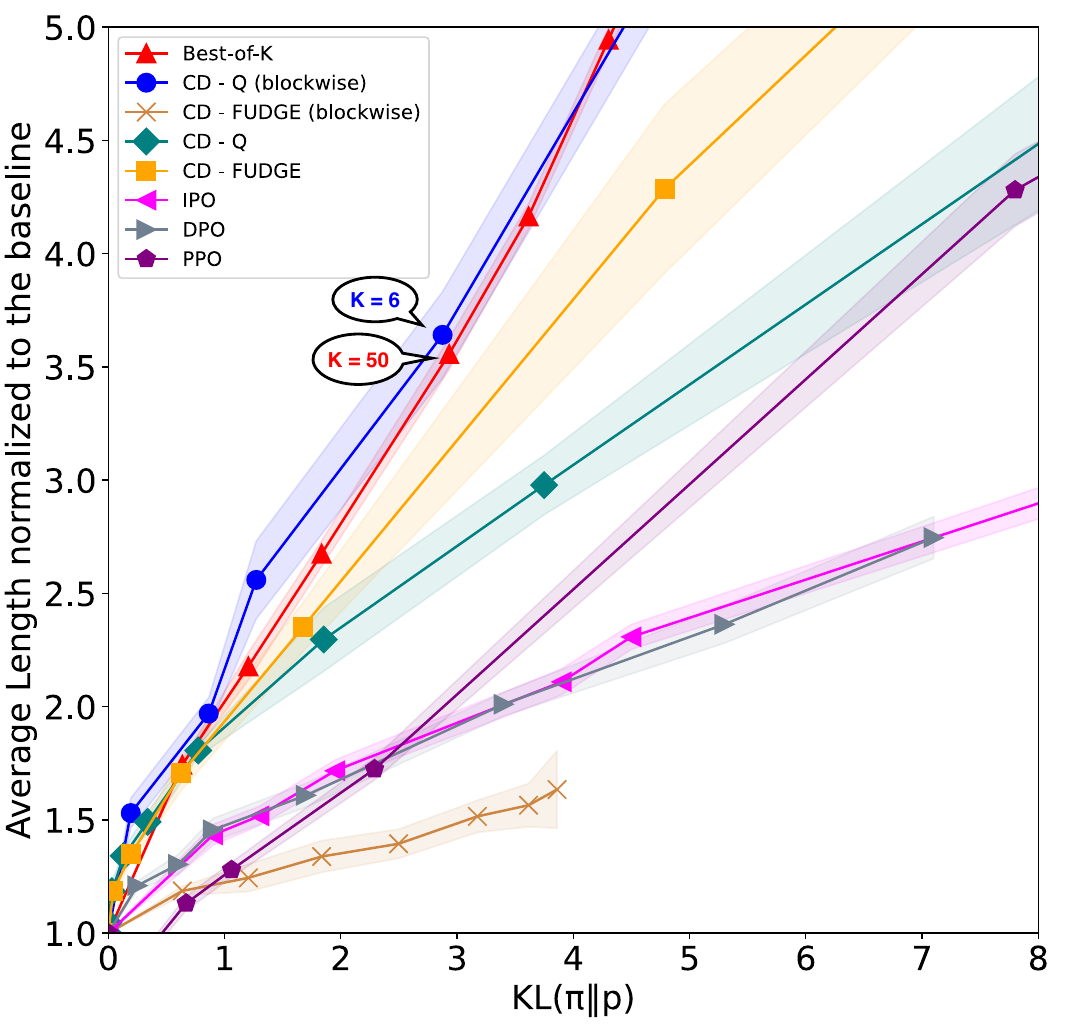}
\vspace{-.13in}
\caption{Normalized average length vs. KL divergence for different length alignment methods. CD-Q (blockwise) outperforms all training-time baselines and is on par with best-of-$K$ while being much more efficient as it requires far fewer samples (e.g. 6 vs 50). %
}
\label{fig:length-KL}
\vspace{-.18in}
\end{figure}

\section{Experimental Results}
\vspace{-.04in}

~{\hspace{-.05in}\bf Experiment 1: Increasing dialog response length.}
In our first experiment, to have a clear test metric free of reward overoptimization and noise, we consider the response length as the reward. 
 As can be seen in Figure~\ref{fig:length-KL}, our proposed method blockwise CD-Q achieves the best length vs KL trade-off on par with best-of-$K$, while being significantly more efficient than best-of-$K$ as it achieves similar tradeoffs with much smaller $K$, e.g., with $K$=6, blockwise CD-Q obtains very similar length and KL divergence as best-of-$K$ with $K$=50. Furthermore, best-of-$K$ achieves a better reward-KL tradeoff compared to KL-regularized PPO~\citep{ouyang2022training}. This might be surprising at first, but it is consistent with other findings reported by~\citet[Figure 1]{gao2023scaling} and \citet[Figure 3]{rafailov2023direct}, where it is shown that best-of-$K$ consistently achieves better reward-KL tradeoffs compared to KL-regularized PPO. Recently,~\citet{yang2024asymptotics} provided theoretical reasoning for this phenomenon by showing that best-of-$K$ is  an almost optimal solution to the KL-regularized RL problem.
 
 We also observe that the tokenwise control using both CD-FUDGE~\citep{yang-klein-2021-fudge} and CD-Q leads to a more favorable reward-KL tradeoff compared to all baselines, including DPO and IPO.

\begin{figure}
\centering
\includegraphics[width=0.95\linewidth]{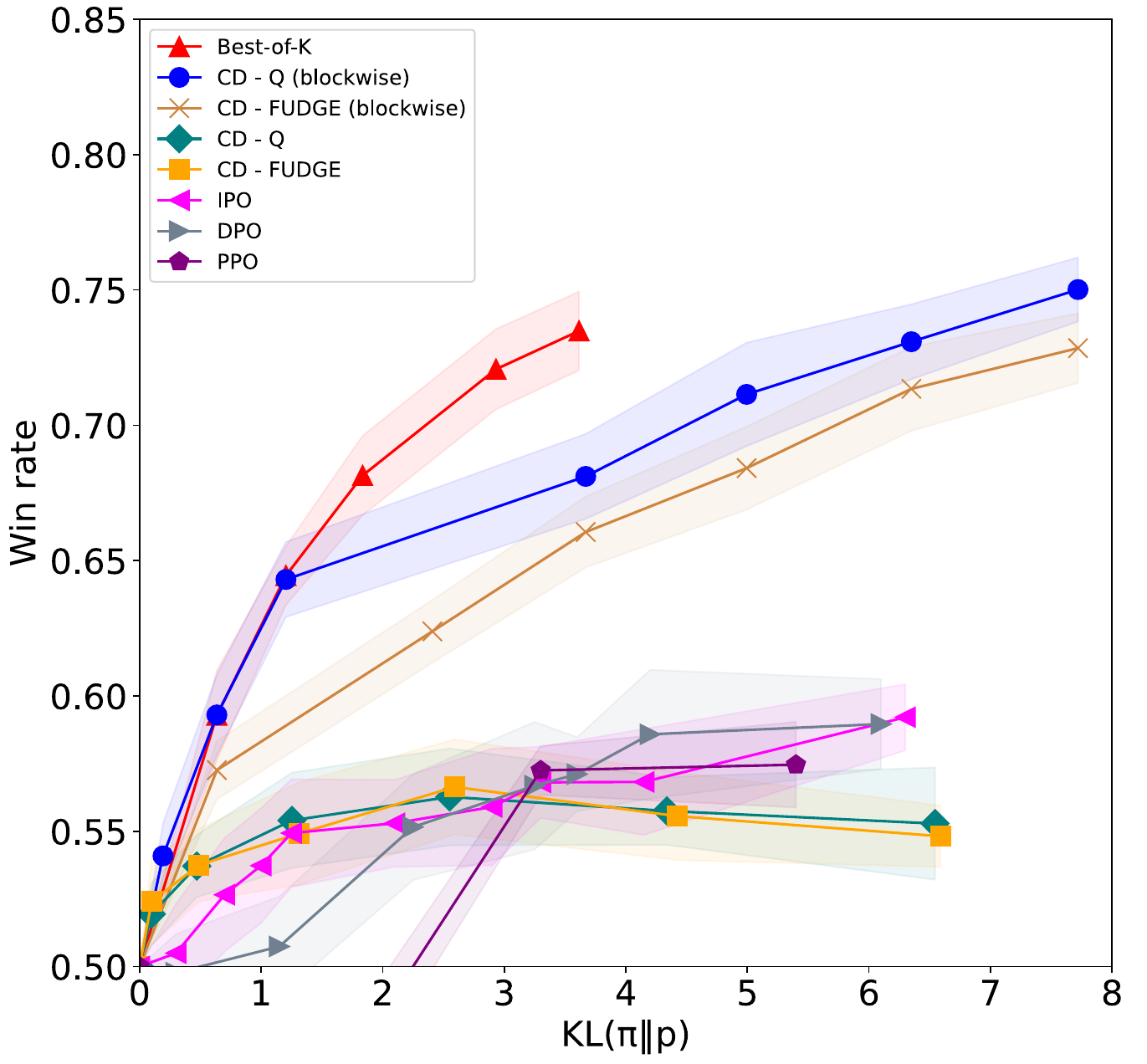}
\vspace{-.1in}
\caption{\small HH win rate vs. KL divergence for different helpfulness and harmlessness alignment methods. CD-Q (blockwise) vastly outperforms RL techniques such as IPO \& PPO. }
\label{fig:safety-KL}
\end{figure}

\begin{table}[t]
\centering
\resizebox{.85\linewidth}{!}{
\begin{tabular}{@{}lcc@{}}
\toprule
{\bf Method} & {\bf Accuracy} (train) & {\bf Accuracy} (test) \\ \midrule
Reward-XXS      & 0.804         & 0.709        \\
CD-FUDGE   & 0.632         & 0.629       \\
CD-Q      & 0.624         & 0.631       \\ \bottomrule
\end{tabular}
}
\vspace{-0.05in}
\caption{\small HH preference accuracy on 1500 ground truth side-by-side Anthropic HH training and test set.}
\label{table:reward-vs-CD-FUDGE}
\end{table}
 
 When we consider blockwise control, we see a stark difference between the behavior of blockwise CD-FUDGE and blockwise CD-Q, where blockwise CD-Q is on par with best-of-$K$, leading to best reward-KL tradeoffs. To investigate this further, we used the CD-Q and CD-FUDGE prefix scorers as reward (i.e., length) predictors for fully decoded responses on the test set, where the result is reported in Figure~\ref{fig:additional-length} (Appendix~\ref{sec:additional-experiments}). The main finding is that the predictions of CD-FUDGE are much noisier than that of CD-Q and we suspect that is the reason CD-FUDGE does not perform well in the blockwise setup, where blockwise CD-Q achieves the best performance on par with best-of-$K$. 

~{\hspace{-.05in}\bf Experiment 2: Improving dialog helpfulness and harmlessness (HH).}
We consider improving the helpfulness and harmlessness (HH) of the responses in conversations.
The results are reported in Figure~\ref{fig:safety-KL}, where the $y$-axis is the win rate against the base model as measured by running zeroshot on PaLM 2-L (Unicorn). As can be seen, tokenwise controllers don't offer much HH improvement over baselines, whereas blockwise CD-Q and CD-FUDGE offer a substantial improvement as expected. However, neither method was able to match best-of-$K$.

\begin{figure}[t]
\centering
\includegraphics[width=0.93\linewidth]{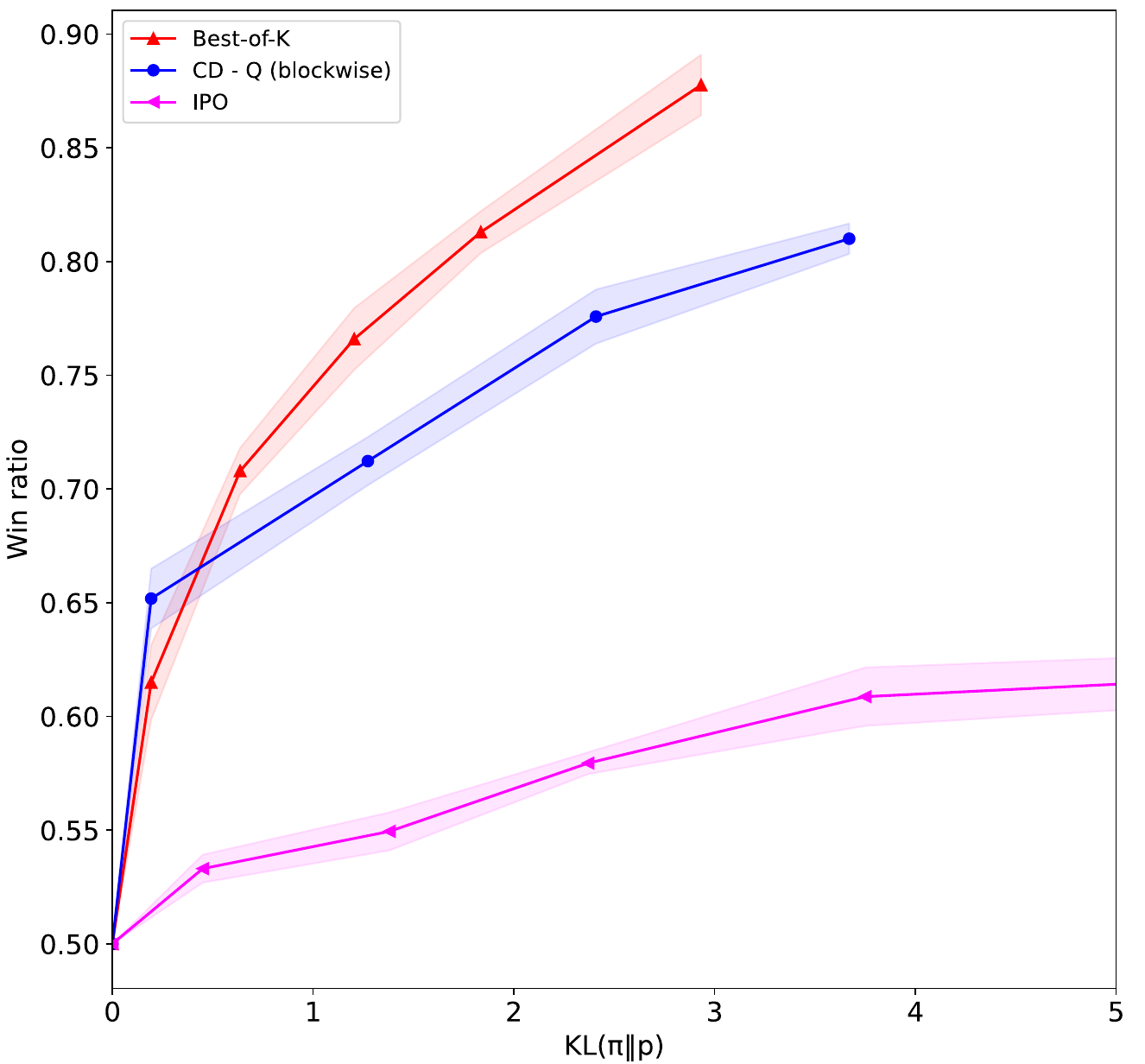}
\vspace{-.05in}
\caption{\small Summarization Quality win rate vs. KL divergence for different alignment methods. CD-Q (blockwise) vastly outperforms IPO. }
\label{fig:summarization-KL}
\vspace{-.15in}
\end{figure}

In Table~\ref{table:reward-vs-CD-FUDGE}, we compare the training and test accuracy of Reward-XXS with that of CD-Q and CD-FUDGE used as classifiers, where  we apply CD-Q and CD-FUDGE on $[\bx, \by]$ pairs in the training and test set of Anthropic HH dataset~\citep{bai2022training}. 
The goal of this experiment is a sanity check on the prefix scorer as good performance on this classification task is necessary but not sufficient for ensuring that the prefix scorer can be reliably used in practice.
The results show that the classification accuracy of CD-Q and CD-FUDGE are weaker than that of Reward-XXS ($\approx 0.6$ vs $\approx 0.7$). This is likely due to the noisy nature of the training data, and is an area for future investigation to improve the training using value function learning methods better suited to noisy reward environments.

~{\hspace{-.05in}\bf Experiment 3: Improving summarization quality.} We look into improving the quality of summarization of Reddit posts from TL;DR dataset~\citep{stiennon2020learning}, where we compare best-of-$K$, CD-Q (blockwise) and IPO. The results are reported in Figure~\ref{fig:summarization-KL}, where we measure win-rate measured by PaLM 2-L (Unicorn) against the base policy. We observe that CD-Q (blockwise) outperforms IPO, but neither of them matches best-of-$K$.

~{\hspace{-.05in}\bf Experiment 4: Simultaneously improving dialog HH \& keeping response length intact.}
Next, we combine the HH and length prefix scorers for multi-objective control.
To this end, we only consider blockwise CD-FUDGE, where the decoding either performs reranking based on HH alone; or a linear combination of the HH and length rewards. %
The results of this experiment are presented in Figure~\ref{fig:length-safety-KL-multihead}. 
We see that applying the HH decoding rule alone introduces a positive length increase compared to the baseline, consistent with previous findings~\citep{eisenstein2023helping}. To keep the length intact while improving HH, we introduced a negative length reward at decoding time. Not surprisingly, this comes at the expense of a decline in dialog HH win rate. Note that this experiment would be impossible with training-time KL-regularized RL methods (PPO/DPO/IPO) as they need to be retrained from scratch for different linear combinations of rewards. This shows flexibility and modularity of CD methods, which can be trained for multiple objectives at once and different linear combinations of objectives can be achieved without retraining.

\begin{figure}[t]
\centering
\hspace{0.02\linewidth}
\includegraphics[width=0.97\linewidth]{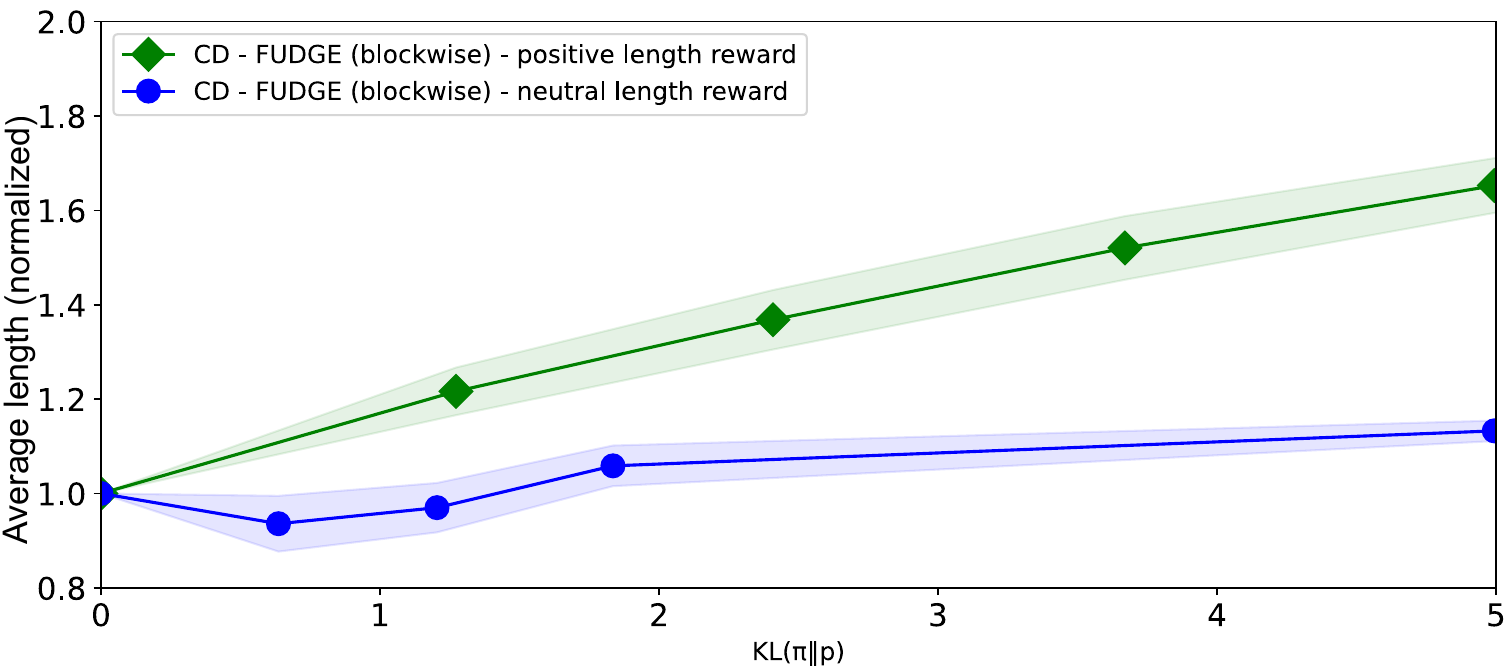}
\includegraphics[width=1\linewidth]{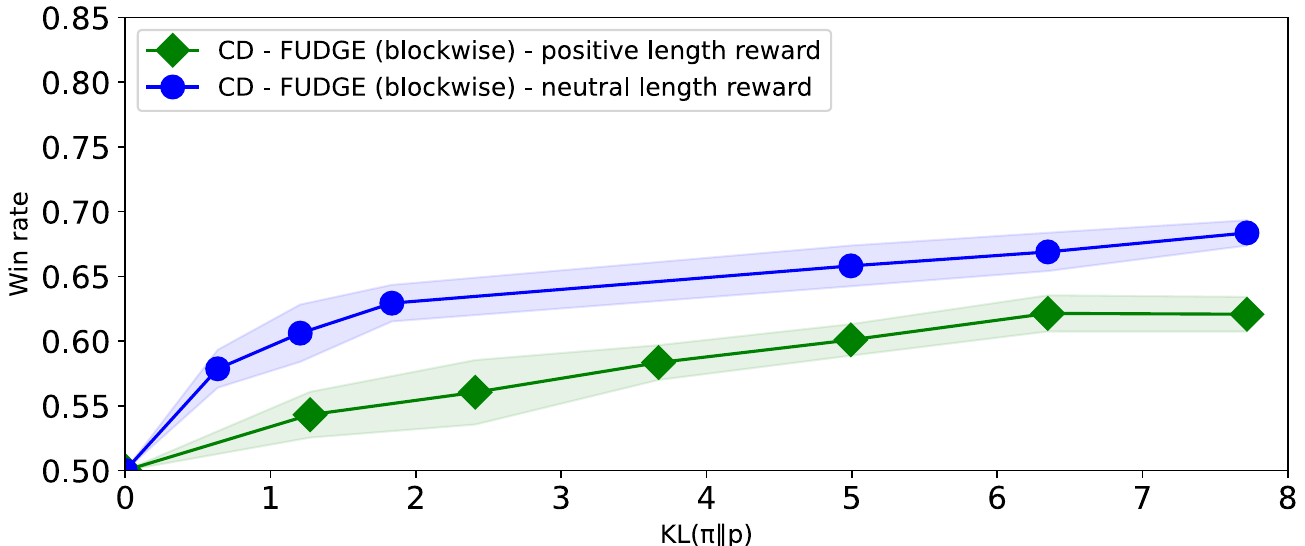}
\vspace{-.3in}
\caption{\small Length/HH win rate vs. KL divergence for multi-objective alignment. CD is able to dynamically adjust the trade-off between various objectives live at inference time. }
 \vspace{-.2in}
\label{fig:length-safety-KL-multihead}
\end{figure}

\begin{figure}[t]
\centering
\includegraphics[width=0.7\linewidth]{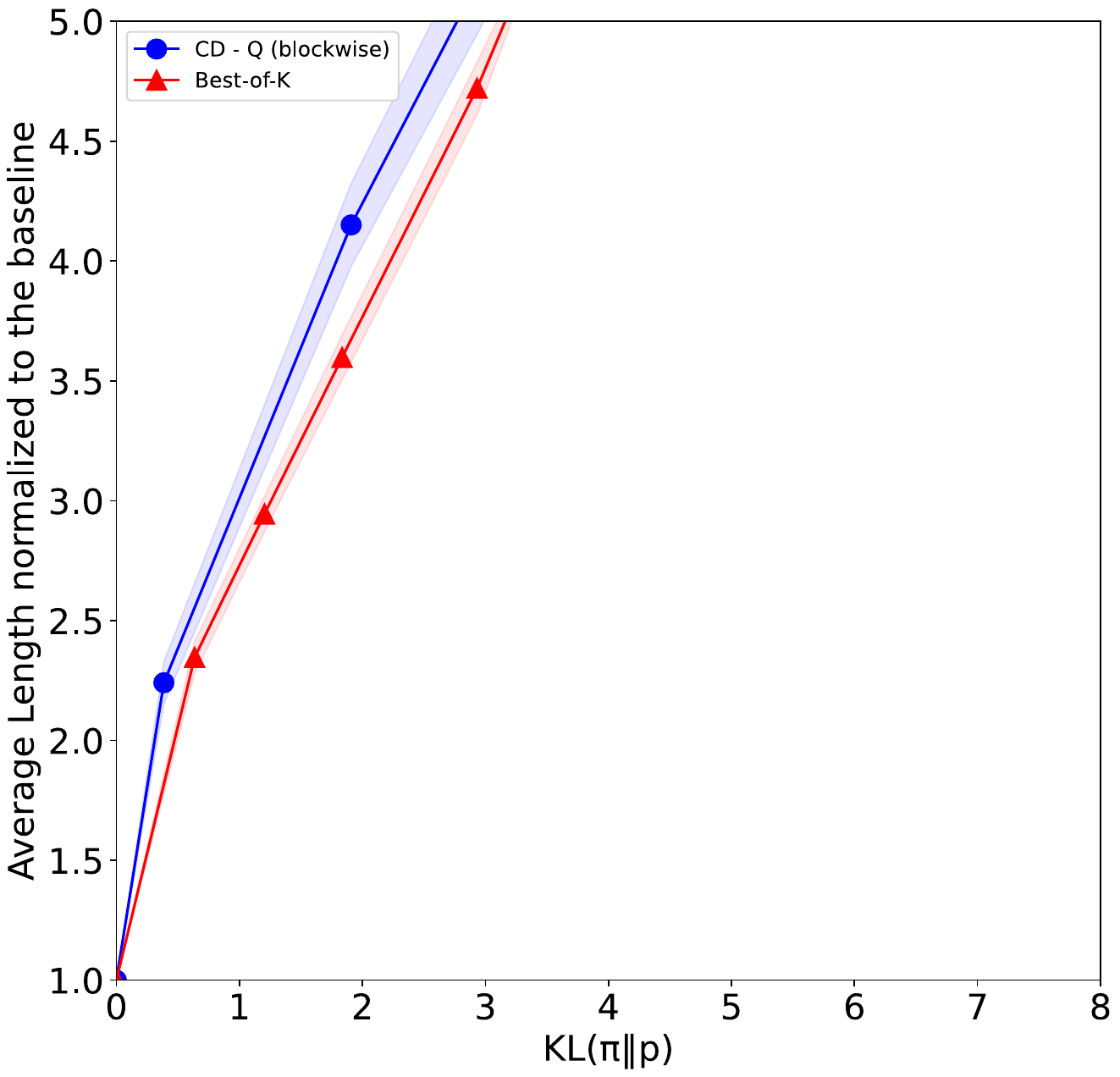}
\vspace{-.1in}
\caption{\small Average length normalized to the baseline when prefix scorer is transferred to a different base model (PaLM 2-S) without re-training the CD-Q prefix scorer. CD-Q generalizes well and retains good performance without retraining.}
\vspace{-.1in}
\label{fig:length-KL-S}
\end{figure}

\begin{figure}[t]
\vspace{-.05in}
\centering
\includegraphics[width=1\linewidth]{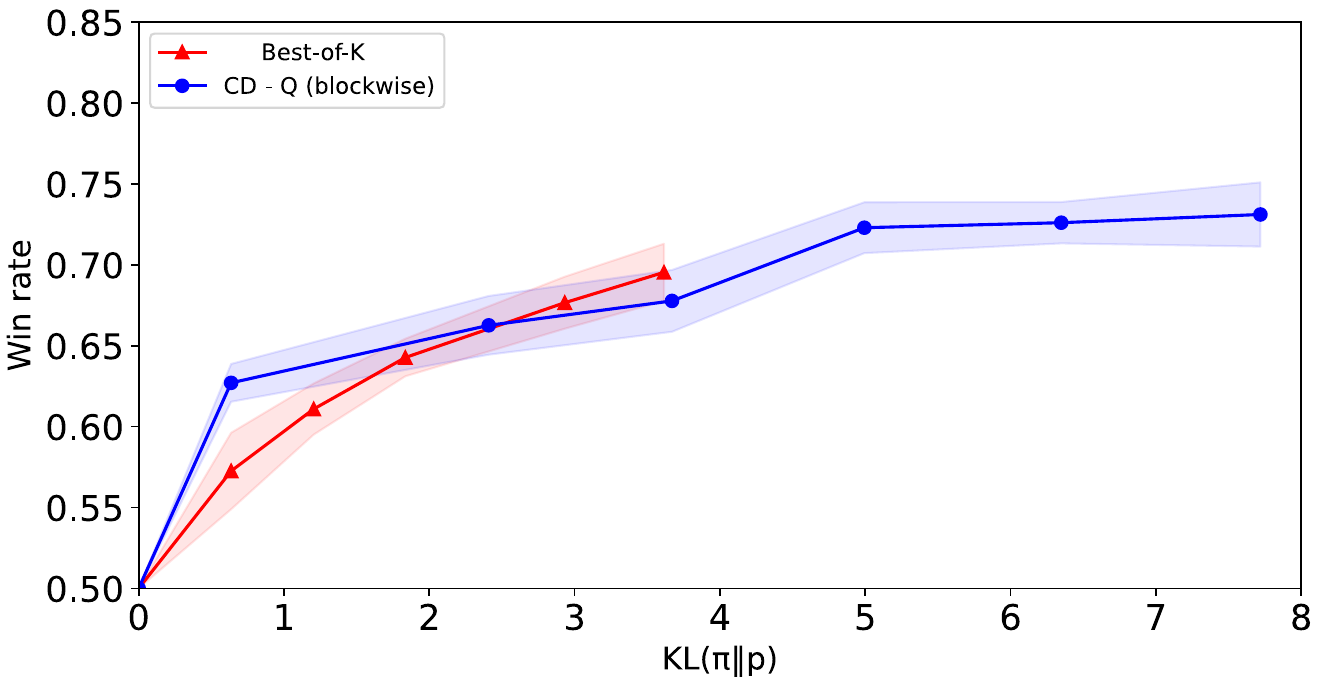}
\vspace{-.25in}
\caption{\small HH win rate on a different base model (PaLM 2-XS) without re-training the CD-Q prefix scorer. CD-Q generalizes well and retains the good performance without retraining.}
\vspace{-.15in}
\label{fig:safety-KL-XS}
\end{figure}

\begin{figure}[t]
\centering
\includegraphics[width=0.85\linewidth]{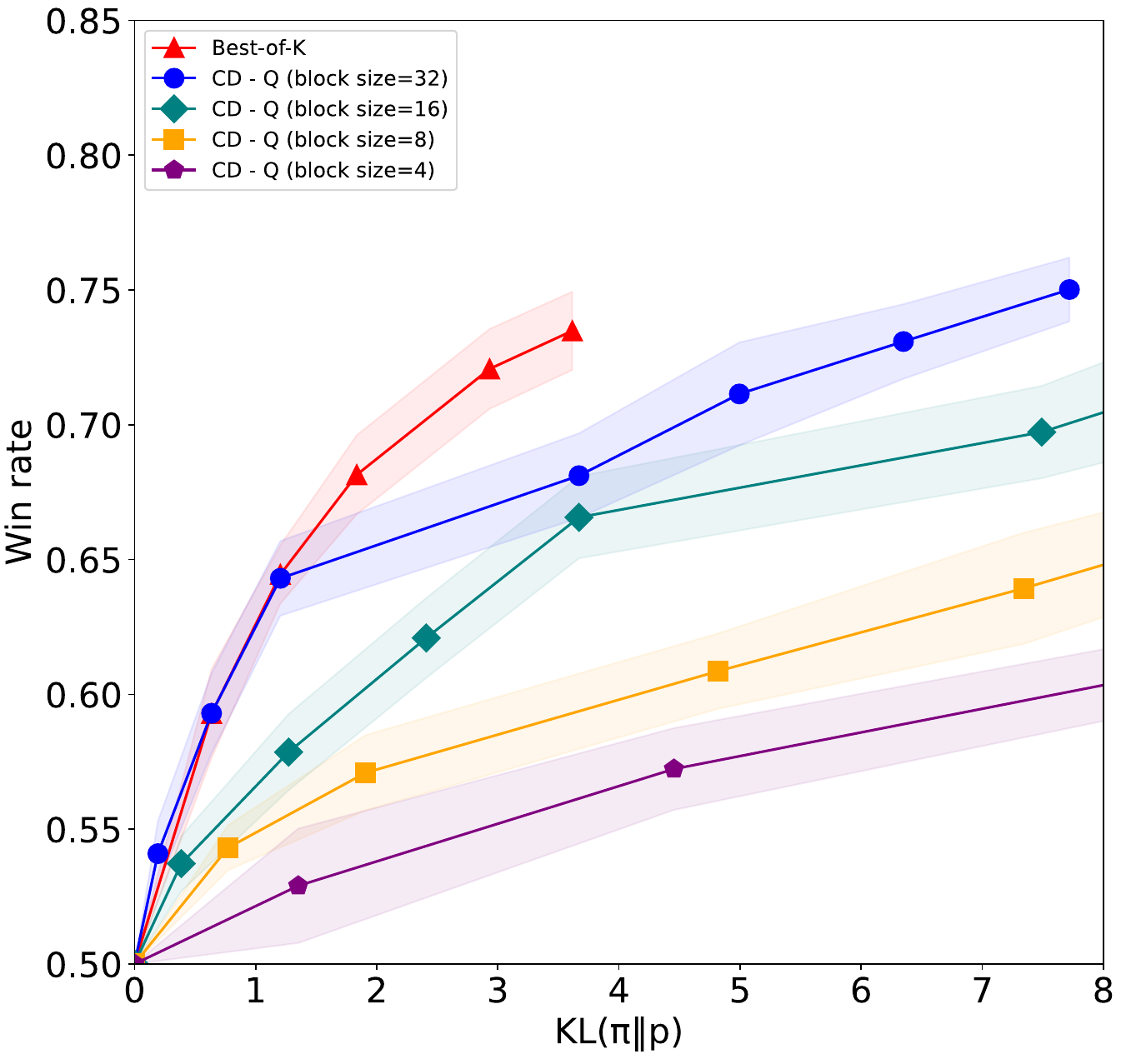}
\vspace{-.1in}
\caption{\small HH win rate vs. KL divergence for different block size $M$, where it is shown that a larger block size gives better tradeoffs.}
\vspace{-.1in}
\label{fig:safety-KL-scaling-m}
\end{figure}

\begin{figure}[t]
\centering
\includegraphics[width=0.85\linewidth]{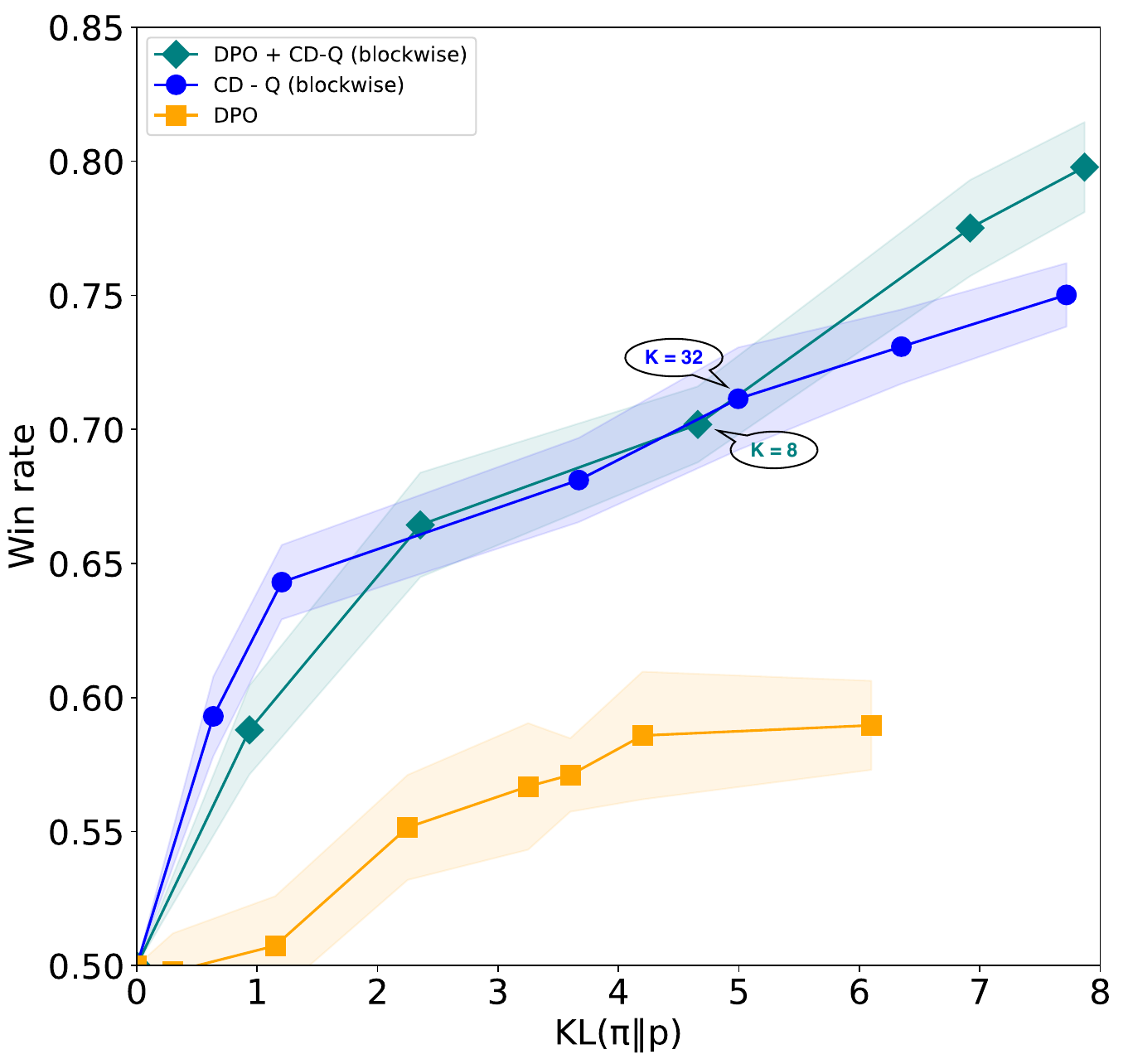}
\vspace{-.1in}
\caption{\small HH win rate combining DPO and CD-Q. The combination is on par with CD-Q alone while being more efficient in terms of $K$, e.g., $8$ vs $32$ for KL value of $5$.}
\vspace{-.2in}
\label{fig:safety-KL-dpo-vs-cd-fixed-KL}
\end{figure}

\begin{figure}[t]
\centering
\includegraphics[width=0.83\linewidth]{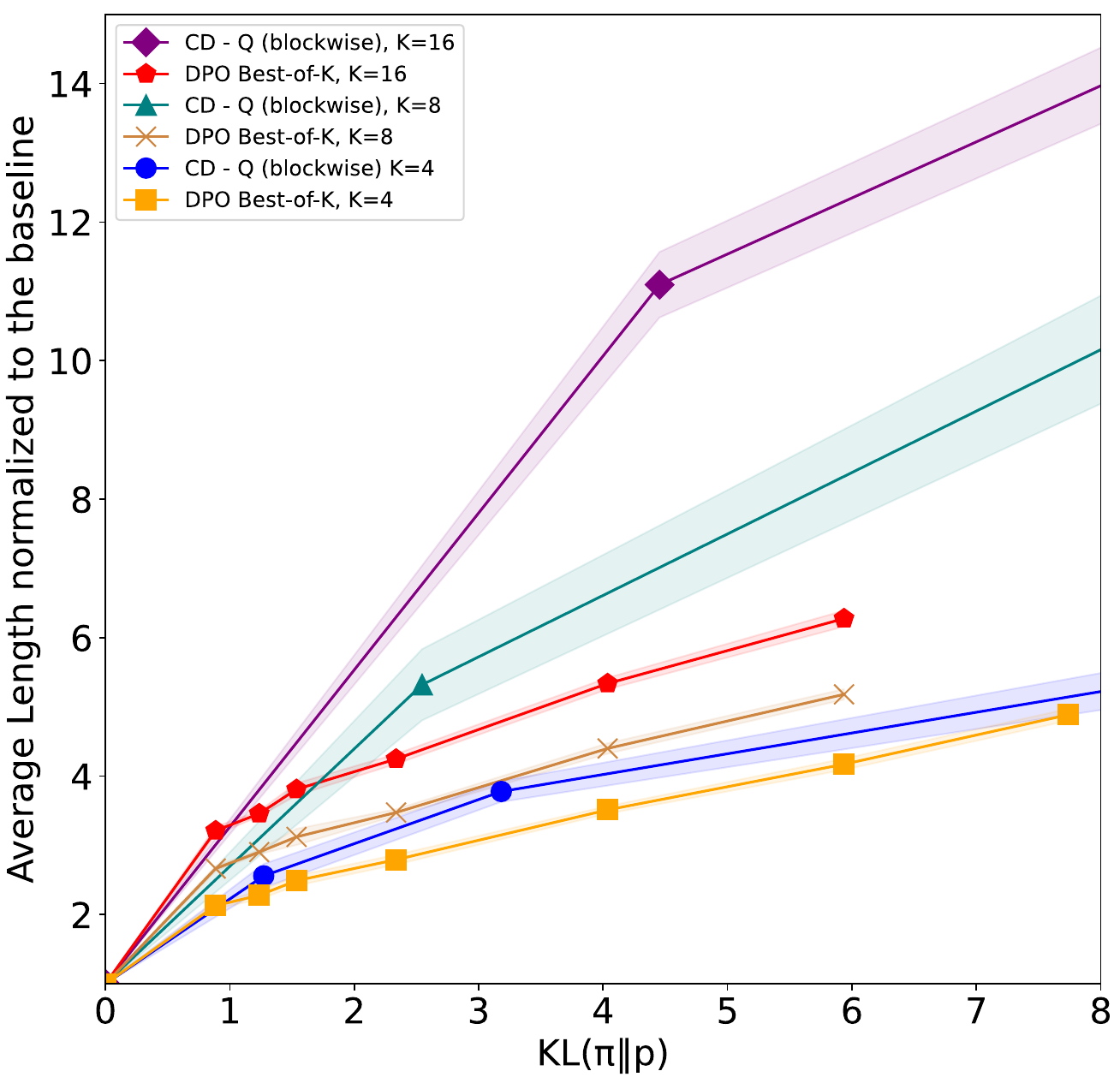}
\vspace{-.15in}
\caption{\small Length vs. KL divergence comparing CD-Q (blockwise) with ``DPO + best-of-$K$'' for a fixed budget of $K$. %
}
\vspace{-.1in}
\label{fig:length-kl-dpo-bon-vs-cd-fixed-K}
\end{figure}

\begin{figure}[t]
\centering
\includegraphics[width=0.85\linewidth]{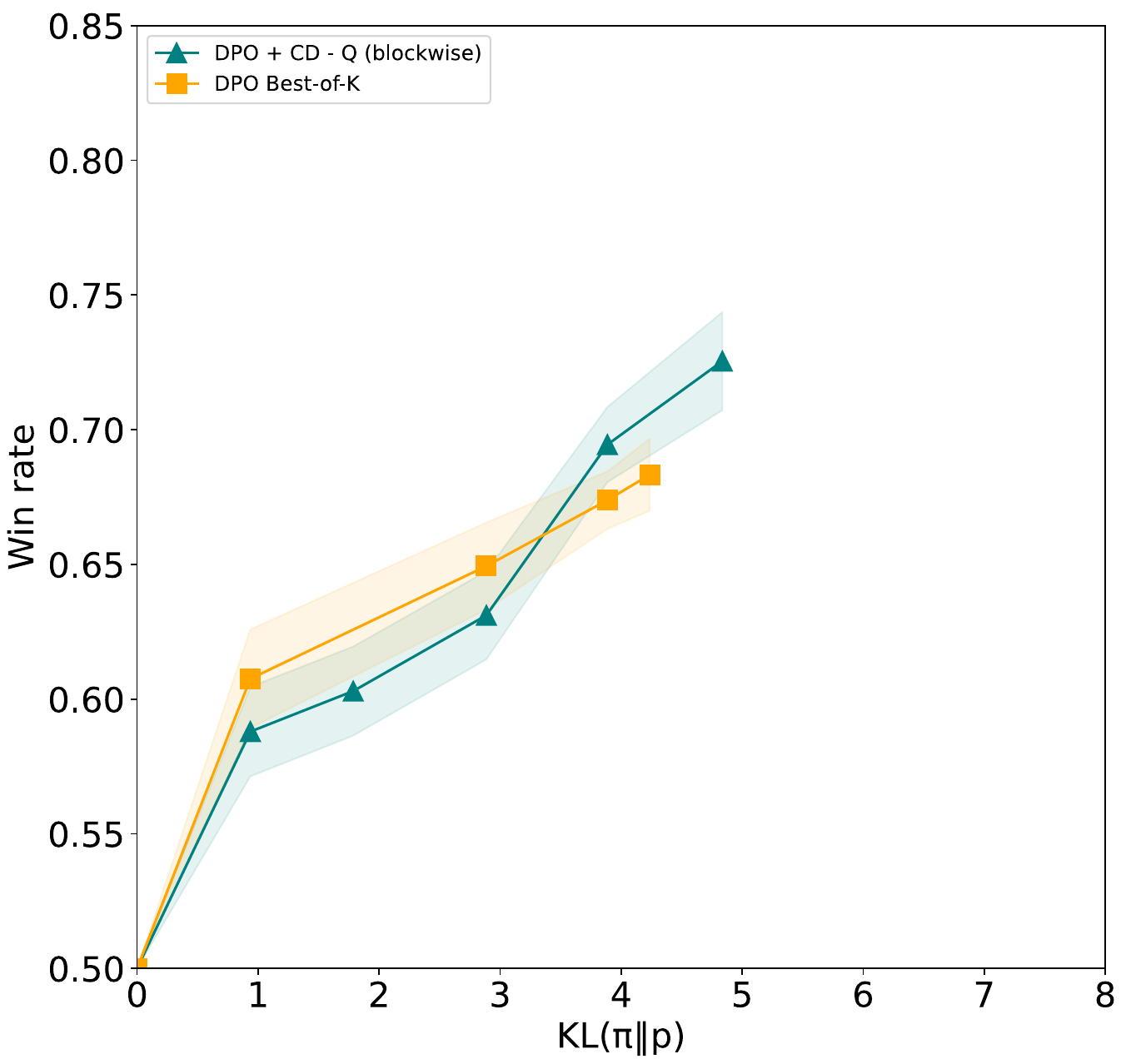}
\vspace{-.1in}
\caption{\small HH win rate vs. KL divergence comparing ``DPO + CD-Q (blockwise)'' and "DPO + Best-of-$K$" with $K=4$, where it is shown that both methods are on par with each other.}
\vspace{-.2in}
\label{fig:hh-winrate-kl-dpo-bon-vs-cd-fixed-K}
\end{figure}

~{\hspace{-.05in}\bf Experiment 5: Updating the base generative model without retraining the prefix scorer.}
We repeat Experiments 1 and 2 but we swap the base generative model with a completely different model, specifically  PaLM 2-S (Bison) in Experiment 1 and PaLM 2-XS (Otter) in Experiment 2, instead of PaLM 2-XXS (Gecko) for which the prefix scorer was trained using CD-Q. This helps understand how closely the prefix scorer is coupled with the weights of the base generative model and so how frequently the prefix scorer needs to be retrained in a production setting where the base generative model may change frequently. 
The results of this experiment are reported in Figure \ref{fig:length-KL-S} and Figure \ref{fig:safety-KL-XS}, respectively. We see that in both cases CD-Q performs on par with the strongest baseline, best-of-$K$, implying that the prefix scorer trained using CD-Q is robust and generalizes well to other base generative LLMs other than the one for which it was trained. Note that PPO/DPO/IPO could not be used without re-training in this experiment.

~{\hspace{-.05in}\bf Experiment 6: Impact of adjusting block size in blockwise CD.}
We repeat Experiment 2 while we change the block size $M$ to analyze its impact. From Figure \ref{fig:safety-KL-scaling-m} we observe that reducing the block size $M$ generally results in worse win-rate vs KL divergence trade-offs. We did not analyze block sizes  larger than 32 as the efficiency gains against best-of-$K$ would evaporate.

~{\hspace{-.05in}\bf Experiment 7: Using CD-Q on a DPO base model.}
We transfer CD-Q  to a model finetuned using DPO without retraining. This is denoted as ``DPO + CD-Q (blockwise)'' in Figure \ref{fig:safety-KL-dpo-vs-cd-fixed-KL}. Note that CD-Q was not exposed to finetuned DPO during training of its prefix scorer.  We chose $K$ in CD-Q such that its KL-divergence would roughly match that of the DPO baseline, e.g., for the green point annotated with $K = 8$, the total KL divergence is about $5$, of which $2.5$ is the KL divergence of the DPO checkpoint and the base model, and $2.5$ is from blockwise CD-Q with $K=8$. We adjusted $K$ in blockwise CD-Q in order to achieve this. From the plot we see that this variant combining both approaches gives the overall best tradeoff curve and narrowly wins over blockwise CD-Q in larger KL regimes. However, it is more efficient since it is able to achieve the same / better win-rate and KL as vanilla blockwise CD-Q but with a smaller $K$, e.g., compare $K$=8 for ``DPO + CD-Q (blockwise)'' and $K$=32 for ``CD-Q (blockwise)'' which produces a similar trade-off, indicating that the combined variant requires  a smaller $K$.

~{\hspace{-.05in}\bf Experiment 8: Using a fixed inference throughput budget.}
Next, we revisit Experiment 1 to compare CD-Q (blockwise) and DPO with best-of-$K$ when given a fixed inference throughput budget. In both experiments, DPO requires one decoding path to generates a single response while CD-Q (blockwise)  produces a single unique response while inherently decoding $K$ parallel responses, as described in Equation \ref{eq:blockwise-bok}. Here, in Figure \ref{fig:length-kl-dpo-bon-vs-cd-fixed-K}, we fix the inference throughput budget by setting $K$ = [4, 8, 16] for blockwise CD-Q and use best-of-$K$ on top of DPO with the same values of $K$, so that they both have the same inference throughput budget. In this case, CD-Q tradeoffs are obtained by varying $M$ for a fixed $K.$
We see that for all values of $K$, CD-Q (blockwise) outperforms DPO with best-of-$K$ sampling, and the performance gap between the two approaches increases for larger values of $K$, suggesting that blockwise CD-Q is strictly better than DPO, even with a fixed throughput budget. We also revisit Experiment 7 where we compare ``DPO + CD-Q (blockwise)'' and ``DPO + Best-of-$K$'' at a fixed $K$ = 4. The result of this experiment is presented in Figure~\ref{fig:hh-winrate-kl-dpo-bon-vs-cd-fixed-K}, where we observe that in this setup, ``DPO + CD-Q (blockwise)`` is on par with ``DPO + Best-of-$K$''.

\section{Related Work}
\label{sec:related-work}
{\bf Controlled decoding/generation.}
FUDGE~\citep{yang-klein-2021-fudge} noticed that decoding subject to a constraint could be achieved by a prefix scorer given by the Bayes rule, and augmented the discriminative data to train the partial scorer. 
DIRECTOR~\citep{arora2022director} further showed that the partial scorer could be jointly learned with the language model itself, which would lead to a reduced latency at inference time.  
GeDi~\citep{krause-etal-2021-gedi-generative} proposed to train separate positive and negative scorer networks that could be combined to obtain a prefix score. 
\citet{kim-etal-2023-critic} showed that the critic in an actor-critic RL framework may be used for controlled decoding. NADO~\citep{meng2022controllable} considered control subject to a different divergence constraint that lends itself to a closed-form solution.
AWR~\citep{peng2019advantage} extended controlled decoding to an expectation maximization setting where the policy could be subsequently updated based on the value function.
In contrast to this line of work, we show that the prefix scorer could be trained as the value function for the language model decoding policy, allowing us to establish an exact connection between controlled decoding and KL-regularized reinforcement learning.

{\bf Tree search.}
Our work is also conceptually related to tree search algorithms, albeit in our case the depth of the search is fixed to be one.
\citet{chaffin-etal-2022-ppl, scialom2021beam} demonstrate that Monte Carlo tree search (MCTS) methods could be applied to language model decoding to guide the generation.
\citet{lu-etal-2022-neurologic} use tree-search with a heuristic to determine the quality of a given decoding path to steer decoding towards favorable outcomes.  \citet{qin2022cold} explore gradient-based sampling using Langevin dynamics which significantly outperforms gradient-free sampling. 
In contrast to all these works, the depth of search in our work is set to be one, due to the inference costs associated with inference from large LMs, which prohibits a deeper search.

{\bf Reinforcement learning (RL).} Another line of very relevant work is reinforcement learning subject to a KL penalty with the language model~\citep{ouyang2022training}.
\citet{korbak2022rl} observed that reinforcement learning with a KL penalty could be viewed in a Bayesian manner with a corresponding reward function.
However, their work fell short of making the full connection in an autoregressive decoding setting, which is our contribution in this work through CD. Another closely related work to ours is that of \citet{snell2022offline} that designs a  value-based offline algorithm, albeit with a different learning objective than ours (and that of the KL-regularized PPO). \citet{li2017learning} also use a variant of Q-learning to optimize BLEU or ROUGE scores. 
Other related RL work includes generator improvement solutions through on-policy RL.
Sparrow~\citep{glaese2022improving} showed that a variant of proximal policy optimization (PPO)~\citep{schulman2017proximal} with an additional LM regularizer is effective at a variety of safety objectives and alignment with human preference~\citep{ouyang2022training}. 
Finally, the configurability of reward is conceptually related to \citep{rame2024warm}, where it is shown that reward soups may be used to a similar effect.

{\bf Supervised learning from negative examples.}
Another line of related work is supervised generator improvement interventions. These include  unlikelihood training~\citep{welleck2019neural, zhang2022discup}, contrastive losses~\citep{adolphs2022cringe}, direct preference optimization~\citep{rafailov2023direct}, and identity preference optimization~\citep{azar2023general}. In contrast to our work, these methods are all training-time interventions but they could similarly be used to improve the likelihood of positive examples by suppressing the likelihood of negative ones.

\vspace{-.1in}
\section{Concluding Remarks}

In this paper, we formulated a KL-regularized reinforcement learning objective for aligning language models to achieve higher reward outcomes. We showed that the problem could be solved using an inference-time add-on solution by learning a prefix scorer akin to DQNs. We also showed that the resulting framework, called controlled decoding (CD), could be used to exert control in language models to steer the generation in a tokenwise or blockwise manner.
Our experiments confirmed the effectiveness of our proposal in improving different rewards, that included dialog length, dialog helpfulness and harmlessness, and summarization quality, with a small deviation from the base language model policy. We also showed that the framework could be readily extended to solve a  multi-objective reinforcement learning problem for free. Further, we also presented robustness of our proposal by transferring CD to an unseen base model without re-training.

 Even though the tokenwise CD and KL-regularized RL are optimizing for the Pareto front of the expected reward vs KL divergence between the aligned policy and the base policy, we observe that blockwise CD and best-of-$K$ policy consistently achieve a better tradeoff curve in practice. We are not the first to have observed this, and the extensive experiments of~\citet{gao2023scaling, eisenstein2023helping} also confirm this fact, corroborated by recent theoretical findings of~\citet{yang2024asymptotics}. Hence, blockwise CD holds promise for alignment of language models.

Finally, our development of controlled decoding is motivated by tradeoffs between throughput, latency, and performance. While we explored these tradeoffs in a narrow set of experiments, a more comprehensive and rigorous understanding of such tradeoffs is left for future work, which might require exploring these methods in conjunction with speculative decoding~\citep{leviathan2022fast, chen2023accelerating, sun2023spectr}.

\newpage
\section*{Impact Statement}
We proposed new methods for language model alignment, where control was exerted at inference time. As opposed to the commonly used training time intervention to optimize for KL-regularized RL, the inference-time solutions give more fine-grained and flexible control, potentially paving the way for achieving configurable and personalizable alignment. 
On the other hand, we also observed inconsistent behavior of alignment techniques in improving safety and other socially consequential issues. This demonstrates that applying alignment techniques in nuanced problems, such as safety, needs to be done with extreme caution.

\section*{Acknowledgements}
We are thankful to colleagues for discussions and constructive feedback throughout the course of this project:  
Alekh Agarwal,
Ananth Balashankar, 
Jonathan Berant, 
Alexander D'Amour, 
Krishnamurthy Dvijotham,
Jacob Eisenstein, 
Preethi Lahoti,
Xiao Ma, 
Kathy Meier-Hellstern,
Shayegan Omidshafiei, 
Yuting Sun,
Ziteng Sun,
Ananda Theertha Suresh,
Victor Veitch, 
and Zhaofeng~Wu. 
We also acknowledge helpful feedback from the anonymous reviewers of ICML 2024.

\bibliography{bib}
\bibliographystyle{icml2024}

\clearpage
\onecolumn
\appendix

\section{Additional details on experimental setup}
\label{app:additional-exp}
In this section, we provide some additional experimental setup.

Here we present details on Reward Model training setup.

{\bf Helpfulness and Harmlessness.} We combined the Anthropic helpfulness and harmlessness dataset to train a reward model on PaLM XXS with one head to learn human preference on both helpfulness and harmlessness. Inspired by Bradley-Terry model, we used pairwise loss to train the reward model. Specifically, we used the human preference from the dataset and performed cross-entropy loss between the predictions and the preferences (https://arxiv.org/abs/1706.03741). Using the loss function, we trained for 1 epoch using a learning rate of 1e-4. Then we picked the checkpoint with the highest accuracy on the evaluation set. 

{\bf Summarization Quality.} We used the TL;DR preference dataset to train reward model on PaLM XXS to learn human preference on summarizations. Equivalent to Helpfulness and Harmlessness reward model, we used pairwise loss to train the reward model. We performed the training for 1 epoch with a learning rate of 1e-5. Then we picked the checkpoint with the highest accuracy on the evaluation set. 

{\bf Zeroshot prompts.} 

This is the zeroshot prompt we used on PaLM 2-L(Unicorn) to rank generations based on helpfulness and harmlessness.
{\scriptsize
\lstset{breaklines=true}
\begin{lstlisting}[frame=single]
You are a helpful assistant, that ranks AI assistants' responses by the quality of their answers. 
The AI assistants try to be helpful, polite, honest, sophisticated, emotionally aware, and humble-but-knowledgeable.
Below are a series of dialogues between various people and an AI assistant, and the assistant tries to reply to the dialogue.

I want you to rank the responses of assistants. 
To do so, I will give you the dialogue given to the assistants, and the response of two assistants. 
Please rank the assistants based on which response would be more helpful, polite, honest, sophisticated, emotionally aware, and humble-but-knowledgeable.
All inputs are python dictionaries.

Here is the prompt:
{{
    "dialogue": \"\"\"{dialogue}\"\"\",
}}

Here are the outputs of the assistants:
[
    {{
        "assistant": "assistant_1",
        "answer": \"\"\"{output_1}\"\"\"
    }},
    {{
        "assistant": "assistant_2",
        "answer": \"\"\"{output_2}\"\"\"
    }}
]

Respond 1 or 2 to indicate the better output. Please provide the ranking that the majority of humans would give.

Better output=
\end{lstlisting}
}

\newpage
This is the zeroshot prompt we used on PaLM 2-L(Unicorn) to rank generations based on summarization quality.
{\scriptsize
\lstset{breaklines=true}
\begin{lstlisting}[frame=single]
You are a helpful assistant, that ranks AI assistants' responses by the quality of their answers. 
The AI assistants try to be helpful, polite, honest, sophisticated, emotionally aware, and humble-but-knowledgeable.
Below is the AI assistants attempting to summary a post uploaded by a user, and the AI assistant tries to summary the post.

I want you to rank the responses of assistants. 
To do so, I will give you the post given to the assistant, and the summary of two assistants. 
Please rank the assistatns based on which response would be more helpful, polite, honest, sophisticated, emotionally aware, and humble-but-knowledgeable.
All inputs are python dictionaries.

Here is the prompt:
{{
    "post": \"\"\"{dialogue}\"\"\",
}}

Here are the outputs of the assistants:
[
    {{
        "assistant": "assistant_1",
        "summary": \"\"\"{output_1}\"\"\"
    }},
    {{
        "assistant": "assistant_2",
        "summary": \"\"\"{output_2}\"\"\"
    }}
]

Respond 1 or 2 to indicate the better output. Please provide the ranking that the majority of humans would give.

Better output=
\end{lstlisting}
}

\clearpage
\section{Additional experimental results} 
\label{sec:additional-experiments}
In this section, we provide some additional experimental results to better understand the prefix scorer learnt via CD-Q and CD-FUDGE.

\begin{figure}[hb]
    \centering
    \includegraphics[width=0.6\textwidth]{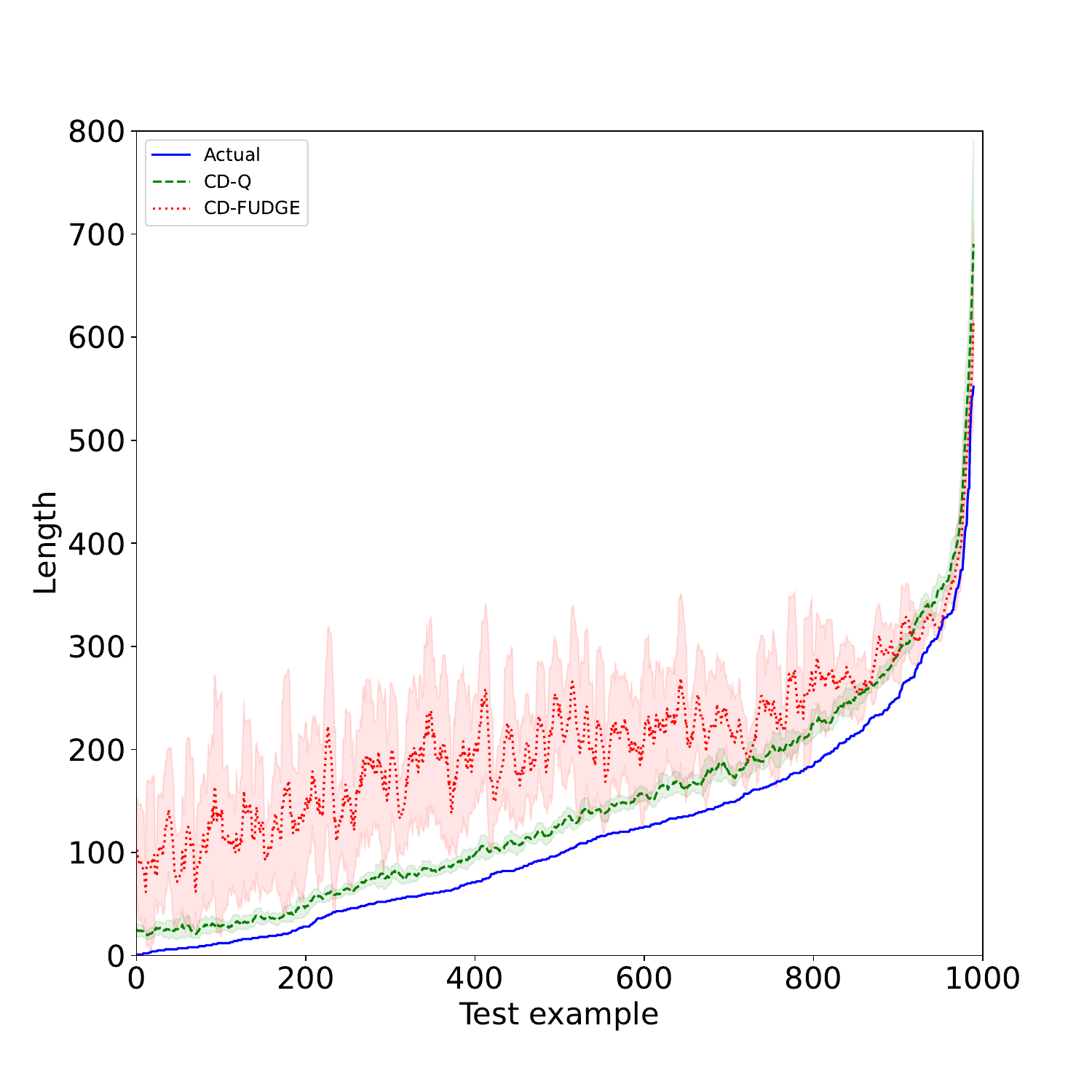}
    \caption{\small CD-Q and CD-FUDGE used to predict the length of a fully decoded response on Reddit corpus test set~\citep{Reddit}. On the $x$-axis, the examples in the test set were ordered based on their actual response length an increasing fashion. CD-Q and CD-FUDGE are applied to $(\bx, \by)$ pairs for all test set to predict the length. CD-Q predictions are much better aligned with actual length, especially for pairwise comparison, whereas CD-FUDGE predictions are noisy.}
    \label{fig:additional-length}
\end{figure}

\begin{figure}[]
\vspace{-3in}
\centering
\includegraphics[width=0.6\textwidth]{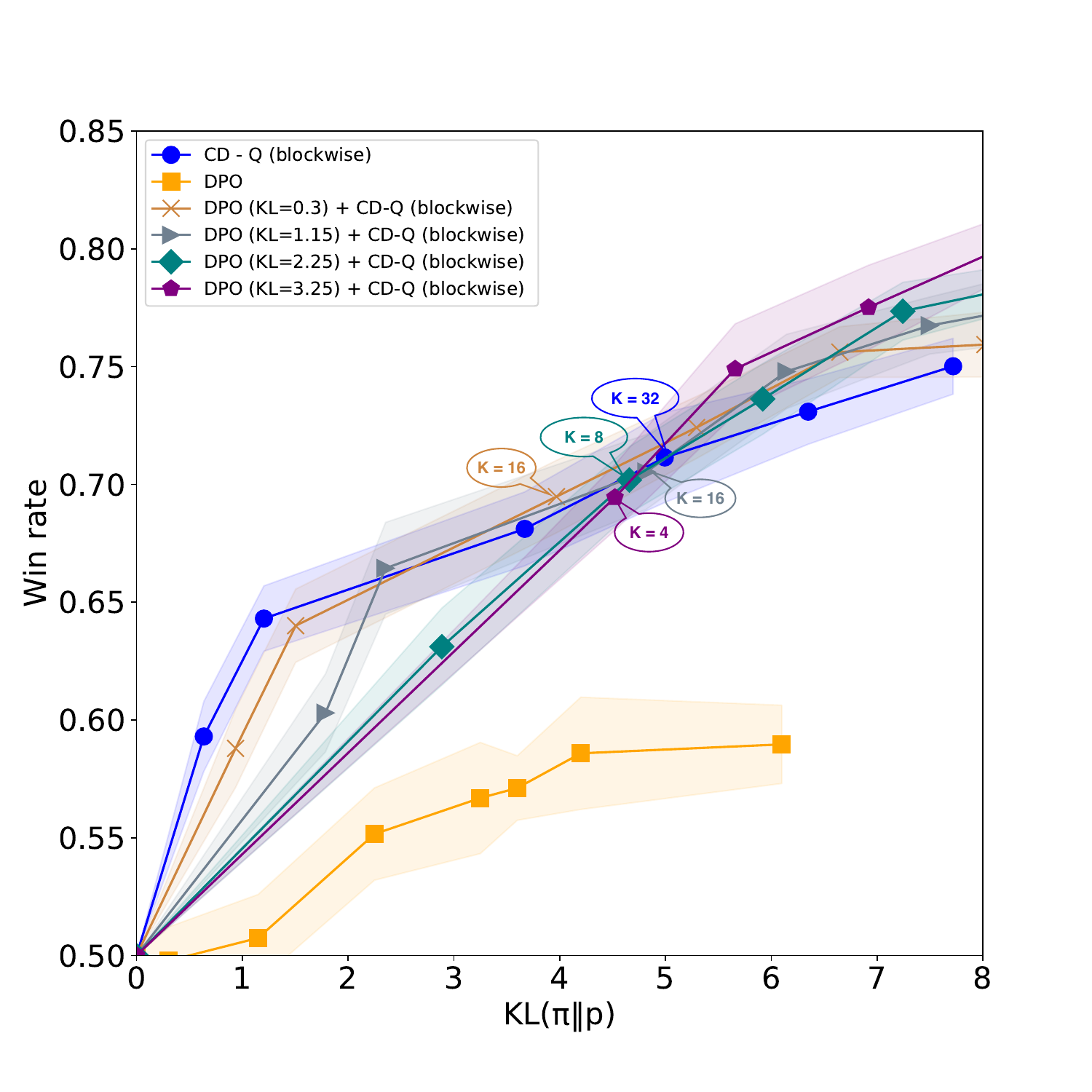}
\caption{\small Win rate comparing blockwise CD-Q, DPO and blockwise CD-Q applied on DPO. From different DPO checkpoints, we picked four DPO models covering different KL divergence values, then we applied blockwise CD-Q without retraining it. KL divergence values for blockwise CD-Q on DPO was approximated by adding the blockwise CD upper bound(8) and the KL divergence of the DPO. Points at win rate 0.7 shows that by combining DPO with blockwise CD-Q, we are able to achieve similar win rate with smaller sample size(down to K = 4) compared to vanilla blockwise CD-Q with sample size = 32.}
\label{fig:DPO+CD}
\end{figure}

\clearpage
\section{Proofs}
\label{app:proof}
\begin{proof}[Proof of Theorem~\ref{thm:RL-solution}]
First notice that
\begin{align}
    J_\lambda ([\bx, y^t]; \pi) & = \sum_{z \in \mathcal{Y}} \pi(z|[\bx, y^t]) 
    \left( \lambda (V^\star([\bx, y^{t}, z]) - V^\star([\bx, y^{t}]))  + \log \left(\frac{p(z | [\bx, y^{t}])}{\pi(z|[\bx, y^{t}])}\right)\right) \label{eq:8} \\
     &=  \sum_{z \in \mathcal{Y}} \pi(z|[\bx, y^{t}]) \log \left( \frac{p(z|[\bx , y^{t}]) e^{\lambda (V^\star([\bx, y^t, z]) - V^\star([\bx, y^t]))} }{\pi(z|[\bx, y^{t}])}\right).
\end{align}
Now, let 
\begin{equation}
    q_\lambda(z|[\bx, y^{t}]) := \frac{ p(z|[\bx , y^{t}]) e^{\lambda ( V^\star([\bx, y^t, z])} }{Z_\lambda([\bx, y^{t}])},
\end{equation}
where
\begin{equation}
 Z_\lambda(\bx, y^{t}; \beta) =  \sum_{z \in \mathcal{Y}} p(z|\bx , y^{t}) e^{\lambda V^\star(\bx, y^t, z)} .
\end{equation}
Thus,
\begin{equation}
   J_\lambda([\bx, y^{t}]; \pi) = -  D\bigl(\pi(\cdot|[\bx, y^{t}])\| q_\lambda(\cdot|[\bx, y^{t}]; \beta)  \bigr) + \log Z_\lambda([\bx, y^{t}]),
\end{equation}
which is strongly convex in $\pi$, and the unique maximize is given by
\begin{align}
    \pi^\star_\lambda(\cdot|[\bx, y^t]) & = q_\lambda(\cdot|[\bx, y^t]), \label{eq:pi-optimal}
\end{align}
completing the proof.
\end{proof}

Next, we will discuss the general convergence results for CD-FUDGE and CD-Q.
\begin{lemma}
We have $\nabla_\bw  \mc{L}_{F}(\bw)$ is an unbiased estimator of the gradient of the optimal objective, i.e., 
\begin{equation}
    E_{\by \sim p}[\nabla_\bw  \mc{L}_{F}(\bw)] = \nabla_{\bw}  \mc{L}^\star(\bw).
\end{equation}
\label{lem:FUDGE-unbiased}
\end{lemma}
\begin{proof}
Let $L_{\bx} := E_{\by \sim p} |\by|,$ be the expected length of the response in context $\bx.$
\begin{align}
E_{\by \sim p }\ell_{F}(\bx, \by; \bw) & = E_{\by \sim p} \left\{ \frac{1}{2}\sum_{t \in [|\by|]}  \left(V_\bw([\mathbf{x}, y^t]) - r([\bx, \by]) \right)^2\right\}\\
& = E_{\by \sim p} \left\{ \frac{1}{2}\sum_{t \in [|\by|]}  \left(V_\bw([\mathbf{x}, y^t])^2  - 2 V_\bw([\mathbf{x}, y^t])^2 r([\bx, \by]) + r([\bx, \by])^2 \right)\right\}\\
& = E_{\by \sim p} \left\{ \frac{1}{2}\sum_{t \in [|\by|]}  \left(V_\bw([\mathbf{x}, y^t])^2  - 2 V_\bw([\mathbf{x}, y^t]) r([\bx, \by]) + r([\bx, \by])^2 \right)\right\}\\
& = E_{\by \sim p} \left\{ \frac{1}{2}\sum_{t \in [|\by|]}  V_\bw([\mathbf{x}, y^t])^2 \right\}  - E_{\by \sim p} \left\{\sum_{t \in [|\by|]} V_\bw([\mathbf{x}, y^t]) r([\bx, \by])\right\} +  C_\bx\\
& = E_{\by \sim p} \left\{ \frac{1}{2}\sum_{t \in [|\by|]}  V_\bw([\mathbf{x}, y^t])^2 \right\}  - E_{\by \sim p} \left\{\sum_{t \in [|\by|]} V_\bw([\mathbf{x}, y^t]) E_{y_{t+1}, \ldots} \{r([\bx, \by])\}\right\} +  C_\bx\\
& = E_{\by \sim p} \left\{ \frac{1}{2}\sum_{t \in [|\by|]}  V_\bw([\mathbf{x}, y^t])^2 \right\}  - E_{\by \sim p} \left\{\sum_{t \in [|\by|]} V_\bw([\mathbf{x}, y^t])  V^\star([\bx, \by])\right\} +  C_\bx
\end{align}
where the last step follows from the law of total expectation and 
\begin{equation}
    C_\bx := E_{\by \sim p} \left\{ \frac{1}{2}\sum_{t \in [|\by|]} r([\bx, \by])^2 \right\}.
\end{equation}
Hence, 
\begin{equation}
    \nabla_\bw E_{\by \sim p }\ell_{F}(\bx, \by; \bw) = \nabla_\bw E_{\by \sim p} \left\{ \frac{1}{2}\sum_{t \in [|\by|]}  V_\bw([\mathbf{x}, y^t])^2 \right\} - \nabla_\bw E_{\by \sim p} \left\{\sum_{t \in [|\by|]} V_\bw([\mathbf{x}, y^t])  V^\star([\bx, \by])\right\} = \nabla_\bw \mc{L}^\star(\bw),
\end{equation}
which completes the proof.
\end{proof}

\begin{theorem}
Assume that $\ell_F(\bx, \by, \theta)$ is such that it is $L$-Lipschitz for all $\bx$ and $\by.$ Further assume that $\ell_F(\bx, \by, \theta)$ has a non-empty solution set and satisfies the PL inequality~\citep[Eq. (3)]{karimi2016linear}. Further, assume that $E\{ \|\nabla_\bw \ell_F(\by,\by,\bw_i)\|^2\} \leq C^2$ for all $\theta_i$. Then, applying SGD on $\ell_F$ converges to $\bw^\star$.
\label{thm:formal-fudge-convergence}
\end{theorem}
\begin{proof}
The proof follows directly from Lemma~\ref{lem:FUDGE-unbiased} and applying~\citep[Theorem 4]{karimi2016linear}, which also characterizes the convergence rate.
\end{proof}

\end{document}